\documentclass[letterpaper]{article} % DO NOT CHANGE THIS
\usepackage[submission]{aaai23}  % DO NOT CHANGE THIS
\usepackage{times}  % DO NOT CHANGE THIS
\usepackage{helvet}  % DO NOT CHANGE THIS
\usepackage{courier}  % DO NOT CHANGE THIS
\usepackage[hyphens]{url}  % DO NOT CHANGE THIS
\usepackage{graphicx} % DO NOT CHANGE THIS
\usepackage{natbib}  % DO NOT CHANGE THIS AND DO NOT ADD ANY OPTIONS TO IT
\usepackage{caption} % DO NOT CHANGE THIS AND DO NOT ADD ANY OPTIONS TO IT
\usepackage{algorithm}
\usepackage{algorithmic}

\usepackage[dvipsnames]{xcolor}
\usepackage[utf8]{inputenc} % allow utf-8 input
\usepackage[T1]{fontenc}    % use 8-bit T1 fonts
\usepackage{url}            % simple URL typesetting
\usepackage{booktabs}       % professional-quality tables
\usepackage{amsfonts}       % blackboard math symbols
\usepackage{nicefrac}       % compact symbols for 1/2, etc.
\usepackage{microtype}      % microtypography

\usepackage{graphicx}
\usepackage{csquotes}
\usepackage{caption}
\usepackage{cancel}
\usepackage{algorithm}
\usepackage{algorithmic}
\usepackage{enumitem}
\usepackage{amsmath,amssymb,amsthm}
\usepackage{pifont}% http://ctan.org/pkg/pifont
\usepackage{thmtools}
\usepackage{cleveref}
\usepackage{shortcuts}
\usepackage{dsfont}
\usepackage{textcomp, gensymb}
\usepackage{wrapfig}
\usepackage{subfiles}
\usepackage[normalem]{ulem}
\usepackage{booktabs}

\usepackage[dvipsnames]{xcolor}
\newtoggle{supplementary}

\newtheorem{theorem}{Theorem}

\newtheorem{lemma}{Lemma}
\newtheorem{proposition}{Proposition}

\newtheorem{definition}{Definition}
\newtheorem{assumption}{Assumption}
\newtheorem{claim}{Claim}

\theoremstyle{remark}
\newtheorem{remark}{Remark}

\usepackage{newfloat}
\usepackage{listings}
\DeclareCaptionStyle{ruled}{labelfont=normalfont,labelsep=colon,strut=off} % DO NOT CHANGE THIS
\floatstyle{ruled}
\newfloat{listing}{tb}{lst}{}
\floatname{listing}{Listing}
{\renewcommand{\arraystretch}{0.9}

\setcounter{secnumdepth}{0} %May be changed to 1 or 2 if section numbers are desired.

\title{Unbalanced CO-Optimal Transport}
\author{
    Quang Huy Tran \textsuperscript{\rm 1,\rm 2}, Hicham Janati \textsuperscript{\rm 3}, Nicolas Courty \textsuperscript{\rm 1}, Rémi Flamary \textsuperscript{\rm 2}, \\ Ievgen Redko \textsuperscript{\rm 4}, Pinar Demetci \textsuperscript{\rm 5, \rm 6}, Ritambhara Singh \textsuperscript{\rm 5, \rm 6}
    }
\affiliations{
    \textsuperscript{\rm 1}Université Bretagne Sud, IRISA \\
    \textsuperscript{\rm 2}CMAP, Ecole Polytechnique, IP Paris \\
    \textsuperscript{\rm 3}LTCI, Télécom Paris, IP Paris \\
    \textsuperscript{\rm 4}Univ. Lyon, UJM-Saint-Etienne, CNRS, UMR 5516 \\
    \textsuperscript{\rm 5}Center for Computational Molecular Biology, Brown University \\
    \textsuperscript{\rm 6}Department of Computer Science, Brown University \\
    \{quang-huy.tran, nicolas.courty\}@univ-ubs.fr, hicham.janati@telecom-paris.fr, \\ remi.flamary@polytechnique.edu, ievgen.redko@univ-st-etienne.fr, \\ pinardemetci@gmail.com, ritambhara@brown.edu
}

\usepackage{bibentry}
\begin{document}

\maketitle

\begin{abstract}
	Optimal transport (OT) compares probability distributions by computing a meaningful alignment between their samples. CO-optimal transport (COOT) takes this comparison further by inferring an alignment between features as well. While this approach leads to better alignments and generalizes both OT and Gromov-Wasserstein distances, we provide a theoretical result showing that it is sensitive to outliers that are omnipresent in real-world data. This prompts us to propose unbalanced COOT for which we provably show its robustness to noise in the compared datasets. To the best of our knowledge, this is the first such result for OT methods in incomparable spaces. With this result in hand, we provide empirical evidence of this robustness for the challenging tasks of heterogeneous domain adaptation with and without varying proportions of classes and simultaneous alignment of samples and features across single-cell measurements.
	%Interestingly, our experiments indicate that owing to the
	%additional Kullback-Leibler divergences, UCOOT is less prone to be stuck in
	%local minima and learns significantly more accurate alignments.
\end{abstract}

%%%%%%%%%%%%%%%%%%%%%%%%%%%%%%%%%%%%%%%%%%%%%%
\section{Introduction}
%%%%%%%%%%%%%%%%%%%%%%%%%%%%%%%%%%%%%%%%%%%%%%

The last decade has witnessed many successful applications of optimal
transport (OT) \cite{Monge81,Kanto42} in machine learning, namely in domain adaptation \cite{Courty16}, generative adversarial networks \cite{Arjovsky17}, 
classification \cite{Frogner15}, dictionary learning \cite{Rolet16}, 
semi-supervised learning \citep{Solomon14}. When the
supports of the probability measures lie in the same ground metric space, it is natural to use the distance defined by the metric
to induce the cost, which leads to the famous Wasserstein distance \citep{Villani03}. When they do not, one can rely on the idea of Gromov-Hausdorff distance \citep{Gromov81} and its equivalent reformulations \citep{Gromov99,Kalton99,Burago01}, and adapt them to the setting of metric measure spaces \citep{Gromov99}. This results in, for example, the Gromov-Wasserstein
(GW) distance \citep{Memoli07,Memoli11,Sturm12}, which has been widely used in many applications, namely in shape matching \citep{Memoli11}, 
comparing kernel matrices \citep{peyre16}, graphs \citep{Vayer19b,Xu19,Xu19b}, 
computational biology \citep{Demetci22}, heterogeneous domain adaptation \citep{Yan18}, 
correspondence alignment \citep{Solomon16}, machine translation \citep{Melis18}.

By construction, the GW distance can only provide the sample alignment that best preserves the intrinsic geometry of the distributions and, as such, compares square pairwise relationship matrices. The CO-Optimal transport (COOT) \citep{Redko20,Chowdhury21b} goes beyond these limits by simultaneously learning two independent (feature and sample) correspondences, and thus provides greater flexibility over the GW distance in terms of usage and interpretability. First, it allows us to measure similarity between arbitrary-size matrices. An interesting use case is, for instance, on tabular data, which are usually expressed as a matrix whose rows represent samples and columns represent features. For the GW distance, 
the similarity or distance matrix (or any square matrix derived from the data)
must be calculated in advance and the effect of the individual variables is lost during this computation. On the other hand, COOT can bypass this step as it can use either the tabular data directly or the similarity matrices as inputs. Second, COOT provides both sample and feature correspondences. These feature correspondences are also interpretable and allow to recover relations between the features of two different datasets even when they do not lie in the same space.

Similar to classical OT, COOT enforces hard constraints on the marginal distributions both between samples and features. These constraints lead to two main limitations: (1) imbalanced datasets where samples or features are re-weighted cannot be accurately compared; (2) mass transportation \emph{must} be exhaustive: outliers, if any, must be matched regardless of the cost they induce. To circumvent these limitations, we propose to relax the mass preservation constraints in the COOT distance and study a broadly applicable and general OT framework that includes several well-studied cases presented in Table~\ref{t:comparisons}.

\paragraph{Related work.}
To relax the OT marginal constraints, a straightforward solution is to control the difference between the marginal distributions of the transportation plan and the data by some discrepancy measure, e.g., Kullback-Leibler divergence. In classical OT, this gives rise to the unbalanced OT (UOT), which was first proposed by~\citep{Benamou03}. The theoretical and numerical aspects of this extension have been studied extensively \citep{Liero18,Chizat18b,Chizat18a,Khiem20} and are gaining increasing attention in the machine 
learning community, with wide-range applications, namely in domain adaptation \citep{Fatras21}, generative adversarial networks 
\citep{Balaji20, Yang19}, dynamic tracking \citep{Lee20}, crowd counting \citep{Ma21}, neuroscience \cite{janati2019group, bazeille2019} or modeling cell developmental trajectories \citep{Schiebinger19}. 
Unbalanced OT and its variants are usually sought for their known robustness to outliers \citep{Mukherjee21,Balaji20,Fatras21}. This appealing property goes beyond classical OT. For instance, to compare signed and non-negative measures in incomparable spaces, unbalanced OT \citep{Liero18} can be blended with the 
$L_p$-transportation distance \citep{Sturm06}, which leads to the Sturm-Entropic-Transport distance \citep{Ponti20}, or with the GW distance, which gives rise to the unbalanced GW (UGW) distance \citep{Sejourne20}. Also motivated by the unbalanced OT, \citep{Zhang21} proposed a relaxation of the 
bidirectional Gromov-Monge distance called unbalanced bidirectional Gromov-Monge divergence.

\begin{table*}[t]
	\centering
		\begin{tabular}{*7l}    
            \toprule
		    & OT & UOT & GW & UGW & COOT & UCOOT \\
            \midrule
            Across spaces & \no & \no & \yes & \yes & \yes & \yes  \\
            Sample alignment & \yes & \yes & \yes & \yes & \yes & \yes \\
            Feature alignment & \no & \no & \no & \no  & \yes & \yes\\
            Robustness to outliers & \no \citep{Fatras21} & \yes \citep{Fatras21} & \no(Prop. \ref{prop:coot-not-robust}) & \yes (Thm. \ref{thm:ucoot_robust}) & \no (Prop. \ref{prop:coot-not-robust})& \yes (Thm. \ref{thm:ucoot_robust})\\
		\bottomrule 
        \hline
		\end{tabular}
		\caption{Properties of different OT formulations generalized by UCOOT. The proposed UCOOT is not only able to learn informative feature alignments, but also robust to outliers. 
    \label{t:comparisons}}
\end{table*}
%%%%%%%%%%%%%%%%%%%%%%%%%%%%%%%%%%%%%%%%%

\paragraph{Contributions.} In this work, we introduce an unbalanced extension of COOT called ``Unbalanced CO-Optimal transport'' (UCOOT). UCOOT -- defined for both discrete and continuous data -- is a general framework that encompasses all the OT variants displayed in Table \ref{t:comparisons}. Our main contribution is to show that UCOOT is provably robust to both samples and features outliers, while its balanced counterpart can be made arbitrarily large with strong enough perturbations. To the best of our knowledge, this is the first time such a general robustness result is established for OT across different spaces. Our theoretical findings are showcased in unsupervised heterogeneous domain adaptation and single-cell multi-omic data alignment, demonstrating a very competitive performance.

\paragraph{Notations.} For any integer $n \geq 1$, we write $[|1, n|] := \{1,...,n\}$. Given a Polish space $\mathcal X$, we denote $\mathcal M^+(\mathcal X)$ the set of nonnegative and finite Borel 
measures over $\mathcal X$. For any $\mu \in \mathcal M^+(\mathcal X)$, we denote 
its mass by $m(\mu) := \mu(\mathcal X)$.
%The support of $\mu \in \mathcal M^+(\mathcal X)$, 
% denoted by $\text{supp}(\mu)$, is the smallest closed set $E \subset \mathcal X$ 
% such that $\mu(\mathcal X \setminus E) = 0$.
Unless specified otherwise, we always consider fully supported measures, 
i.e. $\text{supp}(\mu) = \mathcal X$, for any measure $\mu \in \mathcal M^+(\mathcal X)$. 
The product measure of two measures $\mu$ and $\nu$ is defined as: 
$d (\mu \otimes \nu)(x,y) := d\mu(x) d\nu(y)$.
% Given a Borel measurable map $T: \mathcal X \to \mathcal Y$ and a measure $\mu \in \mathcal{M}^+(\mathcal X)$,
% we define the push-forward (or image) measure $T_{\#}\mu \in \mathcal{M}^+(\mathcal Y)$ is the one which satisfies:
% for every bounded continuous functions $\phi$ on $\mathcal X$,
% \begin{equation}
% \int_{\mathcal X} (\phi \circ T) \; d\mu = \int_{\mathcal Y} \phi \; dT_{\#}\mu.
% \end{equation}
% If $\nu = T_{\#} \mu$, 
% we call $T$ a \textit{transport map} from $\mu$ to $\nu$.
Given $\pi \in \mathcal{M}^{+}(\mathcal X \times \mathcal Y)$, we denote 
$(\pi_{\#1}, \pi_{\#2})$ its marginal distributions i.e. $d\pi_{\#1} = \int_{\mathcal Y} d\pi$ and
$d\pi_{\#2} = \int_{\mathcal X} d\pi$.
% \iffalse Given two probability measures 
% $(\mu, \nu) \in \mathcal{M}^+_1(\mathcal X) \times \mathcal{M}^+_1(\mathcal Y)$, 
% denote \ievgen{U is not used in equations, even though it can simplify them a lot!} $U(\mu, \nu) = \{ \pi \in \mathcal M^+_1(\mathcal X \times \mathcal Y): 
% \pi_{\#1} = \mu, \pi_{\#2} = \nu\}$ the set of admissible transport plans.  Given two measures $\mu \in \mathcal{M}^+(\mathcal X), \nu \in \mathcal{M}^+(\mathcal Y)$ and a Borel measurable map $T: \mathcal X \to \mathcal Y$, we call $T$ a \textit{transport map} from $\mu$ to $\nu$ if for every continuous bounded function $\phi$ defined on $\cY$, we have $\int_{\mathcal X} (\phi \circ T) \; d\mu = \int_{\mathcal Y} \phi \; d\nu$. \fi
For $\mu, \nu \in \cM^+(\cX)$, the Kullback-Leibler divergence is defined by $\kl(\mu\vert\nu) \eqdef \int \frac{d \mu}{ d \nu} \log \frac{d \mu}{ d \nu} \mathrm d\nu - \int \mathrm d\mu + \int \mathrm d\nu$ if $\mu \ll \nu$ and set to $+\infty$ otherwise. Finally, the indicator divergence
$\iota_{=}(\mu \vert \nu)$ is equal to 0 if $\mu = \nu$ and $+\infty$ otherwise.

%%%%%%%%%%%%%%%%%%%%%%%%%%%%%%%%%%%%%%%%%%%%%%%%%%

%%%%%%%%%%%%%%%%%%%%%%%%%%%%%%%%%%%%%%%%%%%%%%%
\section{Unbalanced CO-Optimal Transport (UCOOT)} \label{sec:ucoot}
The ultimate goal behind the CO-Optimal Transport (COOT) framework is the simultaneous alignment of samples \emph{and} features to allow for comparisons across spaces of different dimensions. In this section, we discuss OT formulations including OT, UOT, GW, UGW and COOT, then introduce the proposed UCOOT and show how the aforementioned distances fall into our framework.

% \iffalse After some examples, we present our main theoretical contribution:the robustness to sample and feature outliers. \fi

\paragraph{From sample alignment to sample-feature alignment.}
Let $(\cX_1^s, \mu_1^s)$ and $(\cX_2^s, \mu_2^s)$ be a pair of compact measure spaces such that $\cX_1^s$ and $\cX_2^s$ belong to some common metric space $(\cE, d)$. Classical (unbalanced) optimal transport infers one alignment (or joint distribution) $\pi^s \in \cM^+(\cX_1^{s} \times \cX_2^s)$ with marginals $(\pi^{s}_{\#1}, \pi^{s}_{\#2})$ close to $(\mu_1^s, \mu_2^s)$ according to some appropriate divergence $D$ such that the cost $\int c(x_1, x_2) \mathrm d\pi^{(s)}(x_1, x_2) + D(\pi^{s}_{\#1}\Vert \mu_1^s) + D(\pi^{s}_{\#2}\Vert \mu_2^s)$ is minimal. For instance, in balanced (resp. unbalanced) OT, $D$ corresponds to the indicator divergence (resp. KL divergence or TV). To define a generalized OT beyond one single alignment, we must first introduce a new pair of measure spaces $(\cX_1^{f}, \mu_1^f)$ and $(\cX_2^{f}, \mu_2^f)$.
Intuitively,  the two transport plans that must be inferred: $\pi^s$ across \emph{samples} and $\pi^f$ across \emph{features}, must minimize a cost of the form 
$\iint c((x_1^s, x_1^f), (x_2^s, x_2^f))\mathrm d\pi^s(x_1^s, x_2^s)\mathrm d \pi^f(x_1^f, x_2^f)$ where $c((x_1^s, x_1^f), (x_2^s, x_2^f))$ is the \emph{joint} cost of aligning the sample-feature pairs $(x_1^s, x_1^f)$ and $(x_2^s, x_2^f)$. However, unlike OT, there is no underlying ambient metric space in which comparisons between these pairs are straightforward. Thus, we consider a simplified cost of the form: $c((x_1^s, x_1^f), (x_2^s, x_2^f)) = |\xi_1(x_1^s, x_1^f) - \xi_2(x_2^s, x_2^f)|^p$, for $p \geq 1$ and some scalar functions $\xi_1, \xi_2$ that define the sample-feature interactions. A similar definition was adopted by~\cite{Chowdhury21b} to extend COOT to the continuous setting in the context of hypergraphs. Formally, our general formulation takes pairs of \emph{sample-feature spaces} defined as follows.

\begin{definition}[Sample-feature space]
Let $(\mathcal X^s, \mu^s)$ and $(\mathcal X^f, \mu^f)$
be compact measure spaces, where $\mu^f \in \mathcal M^+(\mathcal X^f)$ and $\mu^s \in \mathcal M^+(\mathcal X^s)$. 
  Let $\xi$ be a scalar integrable function in $L^p(\mathcal X^s \times \mathcal X^f, \mu^s \otimes \mu^f)$. We call the triplet 
  $\mathbb X = ((\mathcal X^s, \mu^s), (\mathcal X^f, \mu^f), \xi)$ a sample-feature space and $\xi$ is called an interaction. 
%   In particular, if $\mathcal X_1 = \mathcal X_2$ and $\mu_1 = \mu_2$, then we call $\mathbb X$ a 
%   homogeneous sample-feature space.
\end{definition} 
%
%%%%%%%%%%%%%%%%%%%%%%%%%%%%%%%%%%%%%%%%%%%%%%%%
\begin{definition}[Generalized COOT]
\label{def:ucoot}
Given two divergences $D_1$ and $D_2$, we define the generalized COOT of order $p$ between $\mathbb{X}_1 = ((\mathcal X_1^s, \mu_1^s), (\mathcal X_1^f, \mu_1^f), \xi_1)$ and 
$\mathbb{X}_2 = ((\mathcal X_2^s, \mu_2^s), (\mathcal X_2^f, \mu_2^f), \xi_2)$ by:
\begin{equation}
\small
\label{eq:ucoot}
  \begin{split}
    %   \ucoot_{\lambda}(\mathbb X_1, \mathbb X_2) := 
  \inf_{\substack{\pi^s \in \mathcal M^+(\mathcal X_1^s \times \mathcal X_2^s) \\
  \pi^f \in \mathcal M^+(\mathcal X_1^f \times \mathcal X_2^f) \\ m(\pi^s) = m(\pi^f)}} 
  &\underbrace{\iint |\xi_1(x_1^s, x_1^f) - \xi_2(x_2^s, x_2^f)|^p \mathrm d\pi^s\mathrm d \pi^f}_{\text{transport cost of sample-feature pairs}} \\
  &+ \underbrace{\sum_{k=1}^2\lambda_k D_k(\pi^s_{\#k} \otimes \pi^f_{\#k} \vert \mu^s_k \otimes \mu^f_k)}_{\text{mass destruction / creation penalty}},
  \end{split}
\end{equation}
for $\lambda_1, \lambda_2 >0$ and $p \geq 1$.
\end{definition}
As the multiplicative nature between $\pi^s$ and $\pi^f$ leads to an invariance by the scaling map $\alpha \mapsto (\alpha \pi^s, \frac{1}{\alpha}\pi^f)$, for $\alpha > 0$, we further impose the equal mass constraint $m(\pi^s) = m(\pi^f)$. 

It is worth mentioning that the formulation \ref{eq:ucoot} is not the only way to relax the marginal constraints. For example, instead of $D_k(\pi^s_{\#k} \otimes \pi^f_{\#k} \vert \mu^s_k \otimes \mu^f_k)$,  one can consider $D_k(\pi^s_{\#k} \vert \mu^s_k) + D_k( \pi^f_{\#k} \vert \mu^f_k)$, or $D_s(\pi^s_{\#1} \otimes \pi^s_{\#2} \vert \mu^s_1 \otimes \mu^s_2)$, for some divergence $D_s$. However, amongst these choices, ours is the only one which can be recast as a variation of the unbalanced OT problem. This allows us to leverage the known techniques in unbalanced OT to justify the theoritical and practical properties, namely Proposition \ref{prop:existence} and Theorem \ref{thm:ucoot_robust} below.

Note that the problem above is very general and can -- with some additional constraints -- recover exact OT, UOT, GW, UGW, COOT (see Table \ref{t:examples}). In particular, if the measures $(\mu^s_1, \mu^s_2)$ and $(\mu^f_1, \mu^f_2)$ are probability measures, then setting $D_1 = D_2 = \iota_=$ leads to the COOT problem first introduced in the discrete case in~\cite{Redko20} and recently generalized to the continuous setting in~\cite{Chowdhury21b}. In this work, we relax the hard constraints and consider a more flexible formulation with the KL divergence:
\begin{definition}[UCOOT]
   We define Unbalanced COOT ($\ucoot$) as in \eqref{eq:ucoot} with $D_1 = D_2 = \kl$. We write $\ucoot_{\lambda}(\mathbb X_1, \mathbb X_2)$ to indicate the UCOOT between two sample-feature spaces $\mathbb X_1$ and $\mathbb X_2$, for a given pair of hyperparameters $\lambda = (\lambda_1, \lambda_2)$.
\end{definition}
While various properties of the divergences $D_k$ have been extensively studied in the context of unbalanced OT by several authors~\citep{Chizat17, Frogner15}, the concept of sample-feature interaction requires more clarification. \iffalse Indeed, except in the setting of~\citep{Chowdhury21b} where $\xi$ corresponds to the hypergraph function, the use of $\xi$ as an interaction is relatively new. \fi Let us consider some simple examples. In the discrete case, we consider $n$ observations of $d$ features represented by matrix $\bA \in \bbR^{n, d}$. In this case, the space $\mathcal X^s$ (resp. $\mathcal X^f$) is not explicitly known but can be characterized by the finite set $[|1, n|]$ (resp. $[|1, d|]$), up to an isomorphism. Assuming that all samples (resp. features) are equally important, the discrete empirical measures can be given by uniform weights $\mu^s = \frac{\mathds 1_{n}}{n}$ (resp. $\mu^f = \frac{\mathds 1_{d}}{d}$). The most natural sample-feature interaction $\xi$ is simply the index function $\xi(i, j) = \bA_{ij}$. In the continuous case, we assume that data stream from a continuous random variable $\ba \sim \mu_s \in \cP(\bbR^d)$ for which an interaction function can be $\xi(\ba, j) = \ba_j$. 

\begin{table*}[t]
  \centering
  \begin{tabular}{*7l}    
      \toprule
		  Requirements & OT & GW & COOT & semi-d. COOT & UCOOT \\
      \midrule
      Shape of inputs & $d_1 = d_2$ & $n_1 = d_1$, $n_2 = d_2$& -- & -- & -- \\
      Coupling constraint & $\pi^f = \mathrm I_{d_1} = \mathrm I_{d_2}$ & $\pi^f = \pi^s$ & -- & -- & -- \\ 
      Scalar function & $\xi(i, j) = \bA_{ij}$ & $\xi(i, j) = \text{dist}( \bA_{i.}, \bA_{j.} ) $ & $\xi(i, j) = \bA_{ij}$ & $\xi(\ba, j) = \ba_j$ & $\xi(i, j) = \bA_{ij}$ \\
      Divergence & $\iota_=$ & $\iota_=$ & $\iota_=$ & $\iota_=$ &   KL \\
		\bottomrule 
		\hline
		\end{tabular}
		\caption{Conditions under which different OT formulations fall within the generalized framework of Definition~\ref{def:ucoot}. ``semi-d'' refers to ``semi-discrete'' setting, where $\mu_s$ is a continuous probability and $\mu_d = \mathds 1_d/d$. Here, $\mathrm I_d$ is the identity matrix in $\mathbb R^d$.
    \label{t:examples}}
\end{table*}

% Recall that when $\mu_k^X, \mu_k^Y$ are probability measures, for $k=1,2$, and $D_1, D_2$ are indicator divergence, it can be shown that COOT \citep{Chowdhury21b} is a special case of UCOOT, where
% \begin{equation} \label{COOT_MK}
%   \coot(\mathbb{X}, \mathbb{Y}) = 
%   \inf_{\substack{\pi_k \in U(\mu^X_k, \mu^Y_k) \\
%   \forall k = 1,2}} \int_{\mathcal S} \vert c_X - c_Y \vert^p \; d \pi_1 \; d \pi_2.
% \end{equation}

%%%%%%%%%%%%%%%%%%%%%%%%%%%%%%%%%%%%%%%%%%%%%
\begin{proposition} \label{prop:existence}
 For any $D_1$, $D_2 \in \{\iota_=, \kl\}$, Problem \ref{def:ucoot} (in Equation \ref{eq:ucoot}) admits a minimizer.
\end{proposition}

\begin{remark}
The existence of minimizer shown in Proposition~\ref{prop:existence} can be extended to a larger family of Csiszár divergences~\citep{Csiszar63}. A general proof is given in the Appendix.
\end{remark}

\paragraph{UCOOT and perfect alignment.} 
Suppose that $\mathbb X_1$ and $\mathbb X_2$ are two finite sample-feature spaces such that $(\mathcal X^s_1, \mathcal X^s_2)$ and $(\mathcal X^f_1, \mathcal X^f_2)$ have the same cardinality and are equipped with the uniform measures $\mu_1^s = \mu_2^s$, $\mu_1^f = \mu_2^f$. Then $\ucoot_{\lambda}(\mathbb X_1, \mathbb X_2) = 0$ if and only if there exist perfect alignments between rows (samples) and between columns (features) of the interaction matrices $\xi_1$ and $\xi_2$.

\section{Robustness of COOT and UCOOT} \label{sec:robustness}
%\ievgen{do we have any other existing results for sensitivity to outliers for OT in different metric spaces? if not, we can use the following plan of attack: 1) we provide first negative results for sensitivity of coot to outliers; applies to GW to some extent as coot is the generalization of the latter; we further show that ucoot is provably robust to outliers: this is the first result of such kind for OT in different metric spaces and applies to UGW confirming the experimental success of UGW showed in \citep{Sejourne20} regarding the robustness to outliers}

% Recall that the COOT problem reads: for $p \geq 1$,
% \begin{equation} \label{eq:coot}
% %   \ucoot_{\lambda}(\mathbb X_1, \mathbb X_2) := 
%   \inf_{\substack{\pi^s \in U(\mu_1^s, \mu_2^s) \\
%   \pi^f \in U(\mu_1^f, \mu_2^f}} \iint |\xi_1(x_1^s, x_1^f) - \xi_2(x_2^s, x_2^f)|^p \mathrm d\pi^s(x_1^s, x_2^s) \mathrm d \pi^f(x_1^f, x_2^f),
% \end{equation}
% where, given two probability measures $\mu \in \mathcal M^+_1(\mathcal X)$ and $\nu \in \mathcal M^+_1(\mathcal Y)$, we denote $U(\mu, \nu) = \{ \pi \in \mathcal M^+_1(\mathcal X \times \mathcal Y): 
% \pi_{\#1} = \mu, \pi_{\#2} = \nu\}$ the set of admissible transport plans.

% \paragraph{UCOOT is robust to outliers.}
When discussing the concept of robustness, outliers are often considered as samples not following the underlying distribution of the data. In our general context of sample-feature alignments, we consider a \emph{pair} $(x^s, x^f) \in \mathcal X^s \times \mathcal X^f$ to be an outlier if the magnitude of $|\xi(x^s, x^f)|$ is abnormally larger than other interactions between $\mathcal X^s$ and $\mathcal X^f$. As a result, such outliers lead to abnormally large transportation costs $|\xi_1 - \xi_2|$. To study the robustness of COOT and UCOOT, we consider an outlier scenario where the marginal data distributions are contaminated by some additive noise distribution.
\begin{assumption}
\label{assump:robust}
Consider two sample-feature spaces $\bbX_k = ((\mathcal X^s_k, \mu^s_k), (\mathcal X^f_k, \mu^f_k), \xi_k)$, for $k=1,2$. Let $\varepsilon^s$ (resp. $\varepsilon^f$) be a probability measure with compact support $\cO^s$ (resp. $\cO^f$). For $a \in \{s, f\}$, define the noisy distribution $\widetilde{\mu}^a = \alpha_a \mu^a + (1-\alpha_a) \varepsilon^a$, where $\alpha_a \in [0,1]$. We assume that $\xi_1$ is defined on $(\mathcal X^s_1 \cup \cO^s) \times (\mathcal X^f_1 \cup \cO^f)$ and that $\xi_1, \xi_2$ are continuous on their supports. We denote the contaminated sample-feature space by $\widetilde{\mathbb X_1} = ((\mathcal X^s_1 \cup \cO^s, \widetilde{\mu}^s_1), (\cX^f_1 \cup \cO^f, \widetilde{\mu}^f_1), \xi_1)$. Finally, we define some useful minimal and maximal costs:
  \begin{equation*}
  \small
        \begin{cases}
  \Delta_{0} \eqdef& \min_{
  \substack{
       x_1^s \in \cO^s, x_1^f \in \cO^f  \\
       x_2^s \in \cX_2^s, x_2^f \in \cX_2^f
  }}\quad |\xi_1(x_1^s, x_1^f) - \xi_2(x_2^s, x_2^f)|^p \\
  \Delta_{\infty} \eqdef& \max_{
  \substack{
  x_1^s \in \cX_1^s \cup \cO^s, x_1^f \in \cX_1^f \cup\cO^f \\
  x_2^s \in \cX_2^s, x_2^f \in \cX_2^f
  }} \quad|\xi_1(x_1^s, x_1^f) - \xi_2(x_2^s, x_2^f)|^p    \enspace.
  \end{cases}
  \end{equation*}
Here, $\Delta_{0}$ accounts for the minimal deviation of the cost between the outliers and target support, while $\Delta_{\infty}$ is the maximal deviation between the contaminated source and the target.
\end{assumption}
The exact marginal constraints of COOT enforce conservation of mass. Thus, outliers \emph{must} be transported no matter how large their transportation costs are. This intuition is captured by the following result.
\begin{proposition} (COOT is sensitive to outliers)
Consider $\widetilde{\mathds X_1}, \mathds X_2$ as defined in Assumption \ref{assump:robust}. Then:
 \label{prop:coot-not-robust}
\begin{equation}
    \coot(\widetilde{\mathds X_1}, \mathds X_2) \geq (1 - \alpha_s)(1-\alpha_f)\Delta_0.
\end{equation}
\end{proposition}
Whenever the outlier proportion $(1-\alpha_s)(1-\alpha_f)$ is positive, COOT increases with the distance between the supports of the outliers and those of the clean data. Thus, the right hand side of~\eqref{prop:coot-not-robust} can be made arbitrarily large by taking outliers far from the supports of the clean data. 

We can now state our main theoretical contribution. Relaxing the marginal constraints leads to a loss that saturates as outliers get further from the data:

\begin{theorem}
\label{thm:ucoot_robust}
(UCOOT is robust to outliers)
Consider two sample-feature spaces $\widetilde{\mathds X_1}, \mathds X_2$ as defined in Assumption \ref{assump:robust}. Let $\delta \eqdef 2(\lambda_1 + \lambda_2)(1 - \alpha_s\alpha_f)$ and $K = M + \frac{1}{M}\ucoot(\mathbb X_1, \mathbb X_2) +\delta$, where $M= m(\pi^s) = m(\pi^f)$ is the transported mass between clean data. Then:
   \begin{equation*} %\label{eq:ucoot-robust}
    \begin{split}
      \ucoot(\widetilde{\mathbb X_1}, \mathbb X_2)
      &\leq \alpha_s \alpha_f \ucoot(\mathbb X_1, \mathbb X_2) \\
      &+ \delta M \left[ 1 - \exp \left( {- \frac{\Delta_{\infty}(1+M) + K}{\delta M}} \right) \right].
    \end{split}
  \end{equation*}
\end{theorem}
%%%%%%%%%%%%%%%%%%%%%%%%%%%%%%%%%%%%%%%%%
The proof of Theorem~\ref{thm:ucoot_robust} is provided in the Appendix and inspired from~\citep{Fatras21}, but in a much more general setting: (1) it covers both sample and feature outliers and (2) considers a noise distribution instead of a Dirac.
Note that the inequality~\eqref{prop:coot-not-robust} indicates that outliers can make COOT arbitrary large, while UCOOT is upper bounded and discards the mass of outliers with high transportation cost.

\setlength{\columnsep}{10pt}%
\setlength{\intextsep}{0pt}
\begin{wrapfigure}[12]{r}{0.24\textwidth}
  \includegraphics[width=1.1\linewidth]{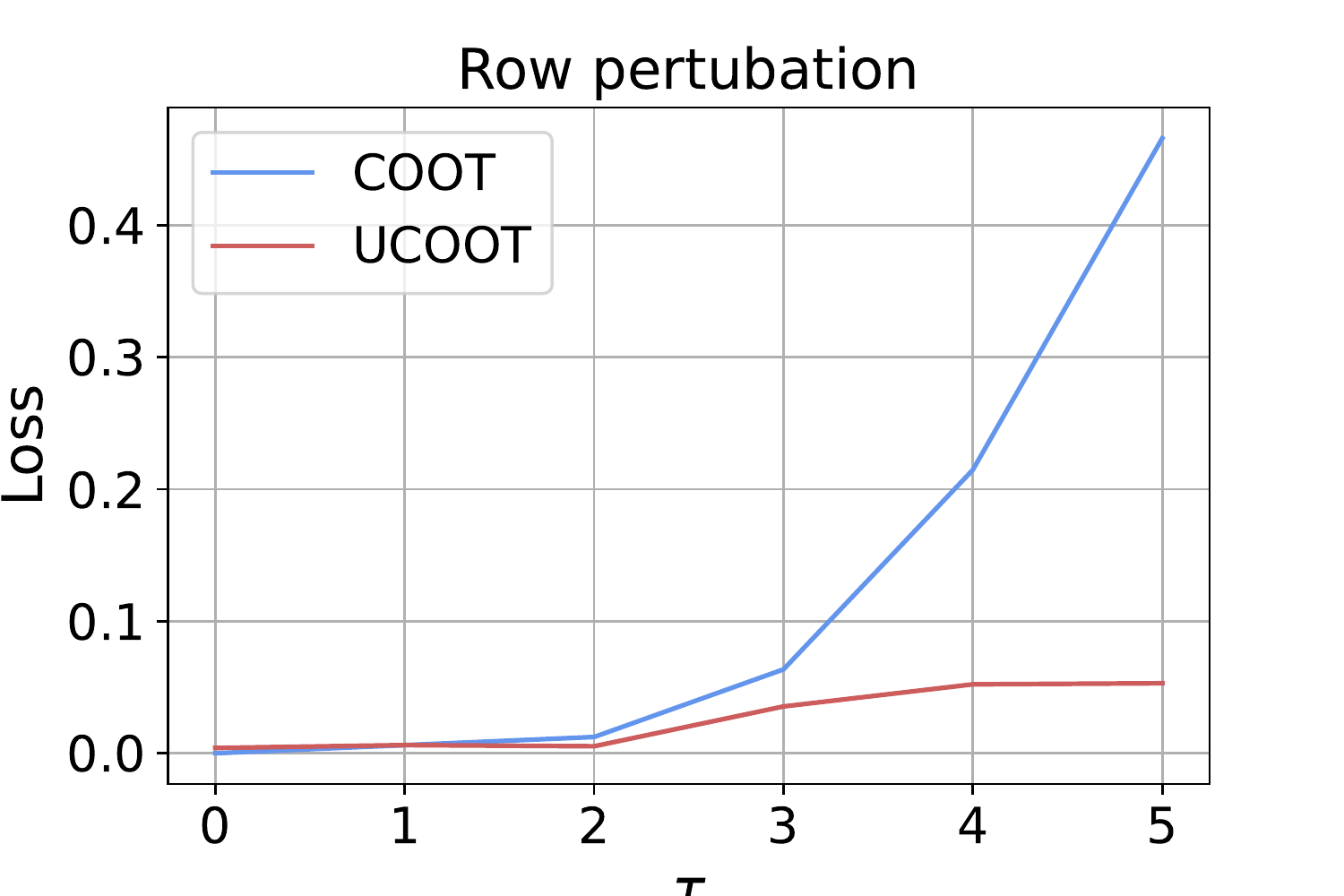}
  \caption{Sensitivity of COOT and UCOOT under the presence of outliers. 
  \label{fig:robust}}
% \end{figure}
\end{wrapfigure}

This is well illustrated in Figure~\ref{fig:robust}, where we simulate outliers by adding a perturbation to a row of the interaction matrix. More precisely, we first generate a matrix $\bA \in \mathbb R^{20,15}$ by $\bA_{ij} = \cos(\frac{i}{20} \pi) + \cos(\frac{j}{15} \pi)$. Then, we replace its last row by $\tau \mathds 1_{15}$, for $\tau \geq 0$. Figure~\ref{fig:robust} depicts COOT and UCOOT between $\bA$ and its modified version as a function of $\tau$. The higher the value of $\tau$, the more likely that the last row contains the interaction of outliers. Consequently, as $\tau$ increases, so does COOT but at a much higher pace, whereas UCOOT remains stable.

It should be noted that, with minimal adaptation, theorem~\ref{thm:ucoot_robust} also holds for the unbalanced GW (UGW) distance. This provides a theoretical explanation of the empirical observation in \citep{Sejourne20} that unlike GW, the UGW distance is also robust to outliers.

\section{Numerical aspects} 
%%%%%%%%%%%%%%%%%%%%%%%%%%%%%%%%%%%%%%%%%%%%%%%%
Solving COOT-type problems, in general, is not trivial. As highlighted in \cite{Redko20}, the balanced case corresponds to a convex relaxation of the bilinear assignment problem, which seeks the pair of permutations minimizing the transport cost. Here we argue that relaxing the marginal constraints makes the problem easier in two different aspects: (1) the obtained problem is easier to solve through a sequence of GPU friendly iterations; (2) regularization leads to lower alignment costs and thus better local minima. In this section, we first describe how to compute UCOOT in practice.

\paragraph{Optimization strategy.}
We consider two tabular datasets $\bA \in \bbR^{n_1, d_1}$ and $\bB \in \bbR^{n_2, d_2}$. Let $u_k$ be the uniform histogram over sample-feature pairs: $u_k \eqdef \frac{1}{n_kd_k}\mathds 1_{n_k} \otimes \mathds 1_{d_k}$, for $k=1, 2$. For the sake of simplicity, we assume uniform weights over both samples and features. Computing UCOOT can be done using block-coordinate descent (BCD) both with and without entropy regularization. More precisely, given a hyperparameter $\varepsilon \geq 0$, discrete UCOOT can be written as:
\begin{equation*} 
\small
\begin{split}
    \label{eq:ucoot-discrete-2}
  &\min_{\substack{\pi^s, \pi^f \\ 
  \iffalse \in \bbR_+^{n_1, n_2} \\ \pi^f \in \bbR_+^{d_1, d_2} \\ \fi m(\pi^s) = m(\pi^f)}} 
  \sum_{i, j, k, l} (\bA_{ik} - \bB_{jl})^2\pi^s_{ij}\pi^f_{kl} + 
  \lambda_1 \kl(\pi^s \mathds 1_{n_1} \otimes \pi^f \mathds 1_{d_1} \vert u_1 )  \\
  &+ \lambda_2 \kl({\pi^s}^\top \mathds 1_{n_2} \otimes {\pi^f}^{\top} \mathds 1_{d_2} \vert u_2) + 
  \varepsilon \text{KL}( \pi^s \otimes \pi^f \vert \mu^s_1 \otimes \mu_2^s \otimes \mu_1^f \otimes \mu_2^f).
\end{split}
\end{equation*}

\begin{algorithm}
    \caption{BCD algorithm to solve UCOOT \label{alg:bcd}}
    \begin{algorithmic}
      \STATE {\bfseries Input:} $\bA \in \bbR^{n_1, d_1}, \bB \in \bbR^{n_2, d_2}$, $\lambda_1, \lambda_2, \varepsilon$
      \STATE Initialize $\pi^s$ and $\pi^f$ 
      \REPEAT
      \STATE Update $\pi^s$ using Sinkhorn or NNPR
      \STATE Rescale $\pi^s = \sqrt{\frac{m(\pi^f)}{m(\pi^s)}} \pi^s$
      \STATE Update $\pi^s$ using Sinkhorn or NNPR  
      \STATE Rescale $\pi^f = \sqrt{\frac{m(\pi^s)}{m(\pi^f)}} \pi^f$
      \UNTIL{convergence}
	\end{algorithmic}
\end{algorithm}
% where recall that the total transported mass $m(\pi)$ is calculated by $\sum_{ij}\pi_{ij}$. The hyperparameters $\lambda_1$ and $\lambda_2$ implicitly control how much of the ``sample-feature'' mass is transported. \huy{this seems a bit redundant because we will rewrite almost the same formulation in the next paragraph.}

The only difference between $\varepsilon = 0$ and $\varepsilon > 0$ lies in the inner-loop algorithm used to update one of transport plans $(\pi^s, \pi^f)$ while the other one remains fixed. Note that both cases allow for  implementations of a scaling multiplicative algorithm that can be parallelized on GPUs. %: Sinkhorn's algorithm ($\varepsilon > 0$) or NNPR ($\varepsilon = 0
%$).  The main BCD scheme is summarized in Algorithm \ref{alg:bcd}.
% ($\varepsilon > 0$ Sinkhorn's algorithm.}
For $\varepsilon > 0$, updating each transport plan boils down to an entropic UOT problem, which can be solved efficiently using the unbalanced variant of Sinkhorn's algorithm~\cite{Chizat18a}.
The main benefit of entropy regularization is to reduce the number of variables from $(n_1 \times n_2) + (d_1 \times d_2)$ to $n_1 + n_2 + d_1 + d_2$. Moreover, by taking $\varepsilon$ sufficiently small, we can recover solutions close to those in the non-entropic case.
\iffalse It should also be noted that, while the proposed algorithm in \citep{Sejourne21} requires mass rescaling within each BCD iteration, 
it is not necessarily the case for the UCOOT. \fi 
For $\varepsilon=0$, the non-regularized UOT problem can be recast as a non-negative penalized regression (NNPR)~\cite{Chapel21}. This problem can be solved using a majorization-minimization algorithm which leads to a multiplicative update on the transport plan. \iffalse The NNPR was proposed and extensively studied by \cite{Chapel21}. \fi
For the sake of reproducibility, we provide the details on the optimization scheme of both algorithms in the Appendix.

\section{Experiments} \label{sec:experiments}
\subsection{Illustration and interpretation of UCOOT on MNIST images}

\begin{wrapfigure}[12]{r}{0.24\textwidth}
% \begin{figure}
%   \centering
  \includegraphics[trim={0.2cm 0.45cm 0.2cm 0.35cm}, clip, width=1.05\linewidth]{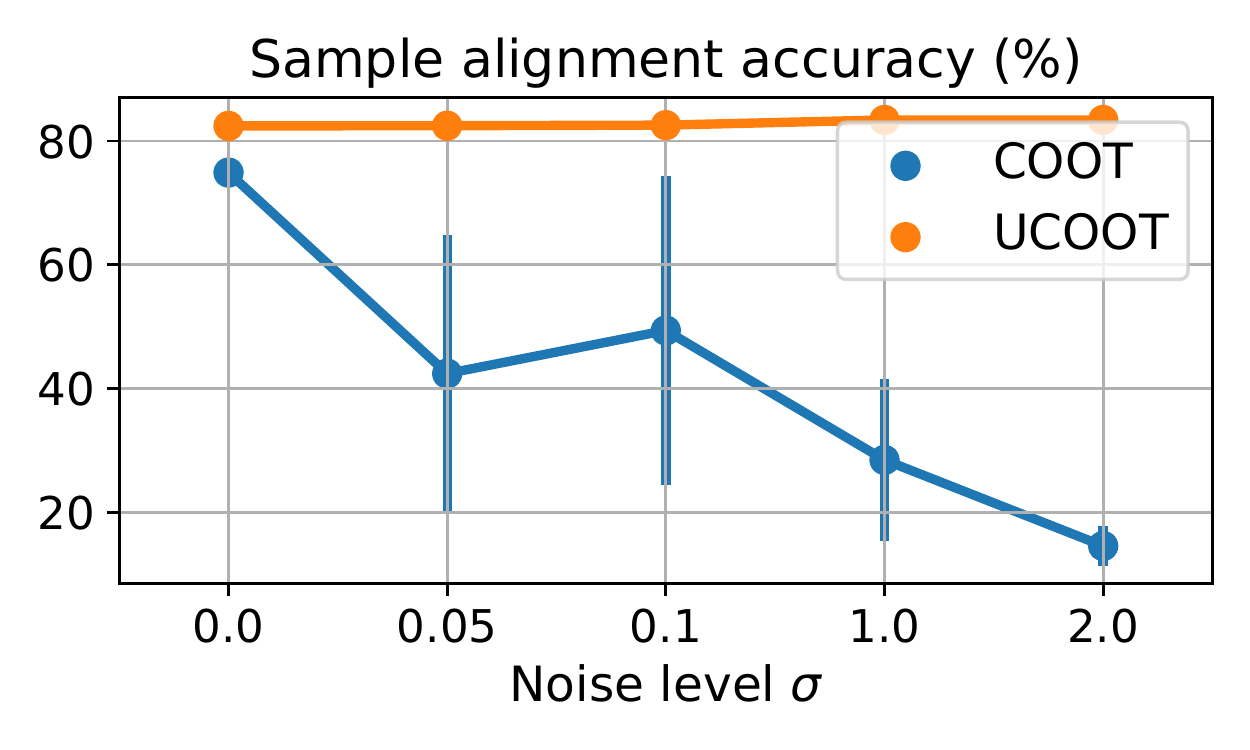}
  \caption{Robustness of UCOOT vs. COOT on MNIST example. The accuracy is averaged on 20 trials.
  \label{f:mnist-sigma}}
% \end{figure}
\end{wrapfigure}

% \begin{wrapfigure}{r}{0.2\textwidth}
%   \includegraphics[width=1.\linewidth]{fig/mnist-sigma.pdf}
%   \caption{Robustness of UCOOT versus COOT on the MNIST example. The accuracy is averaged on 20 trials.
%   \label{f:mnist-sigma}}
% \end{wrapfigure}
We illustrate the robustness of UCOOT and its ability to learn meaningful feature alignments under the presence of both sample and feature outliers in the MNIST dataset. We introduce the feature outliers by applying a handcrafted transformation $\varphi_\sigma$ that performs a zero-padding (shift), a 45\textdegree\ rotation, a resize to (28, 34) and adds Gaussian noise $\cN(0, \sigma^2)$ entries to the first 10 columns of the image.
Figure~\ref{f:mnist-example} (a) shows some examples of original and transformed images. We randomly sample 100 images per class (1000 total) from $X = \text{MNIST}$ and $Y = \varphi_{\sigma}(\text{MNIST})$. Regarding the sample outliers, we add 50 random images with uniform entries in [0, 1] to the target data $Y$. We then compute the optimal COOT and UCOOT alignments shown in Figure~\ref{f:mnist-example} (b) and (c). The flexibility of UCOOT with respect to mass transportation allows it to completely disregard: (1) noisy and uninformative pixels (features), which are all given the same weight as depicted by (b); (2) all the sample outliers of which none are transported as shown by the last blank column of the alignment (c). Moreover, notice how the color-coded input image is transformed according to the transformation $\varphi_{\sigma}$ despite the fact that no spatial information is provided  in the OT problem. On the other hand, a very small perturbation ($\sigma = 0.01)$ is enough for the sample alignment given by COOT to lose its block-diagonal dominant structure (class information is lost), while the UCOOT alignment remains unscathed. One may wonder whether the performance of UCOOT would still hold for different values of $\sigma$. Figure~\ref{f:mnist-sigma} answers this question positively. For $\sigma > 0$, we compute the average accuracy (defined by the percentage of mass within the block-diagonal structure) over 20 different runs. The performance of COOT not only degrades with noisier outliers but is also unstable. By contrast, the accuracy of UCOOT remains almost constant regardless of the level of noise.
\begin{figure*}[t]
  \centering
  \includegraphics[trim={0.5cm 4.5cm 1.5cm 2.4cm}, clip, width=1\linewidth]{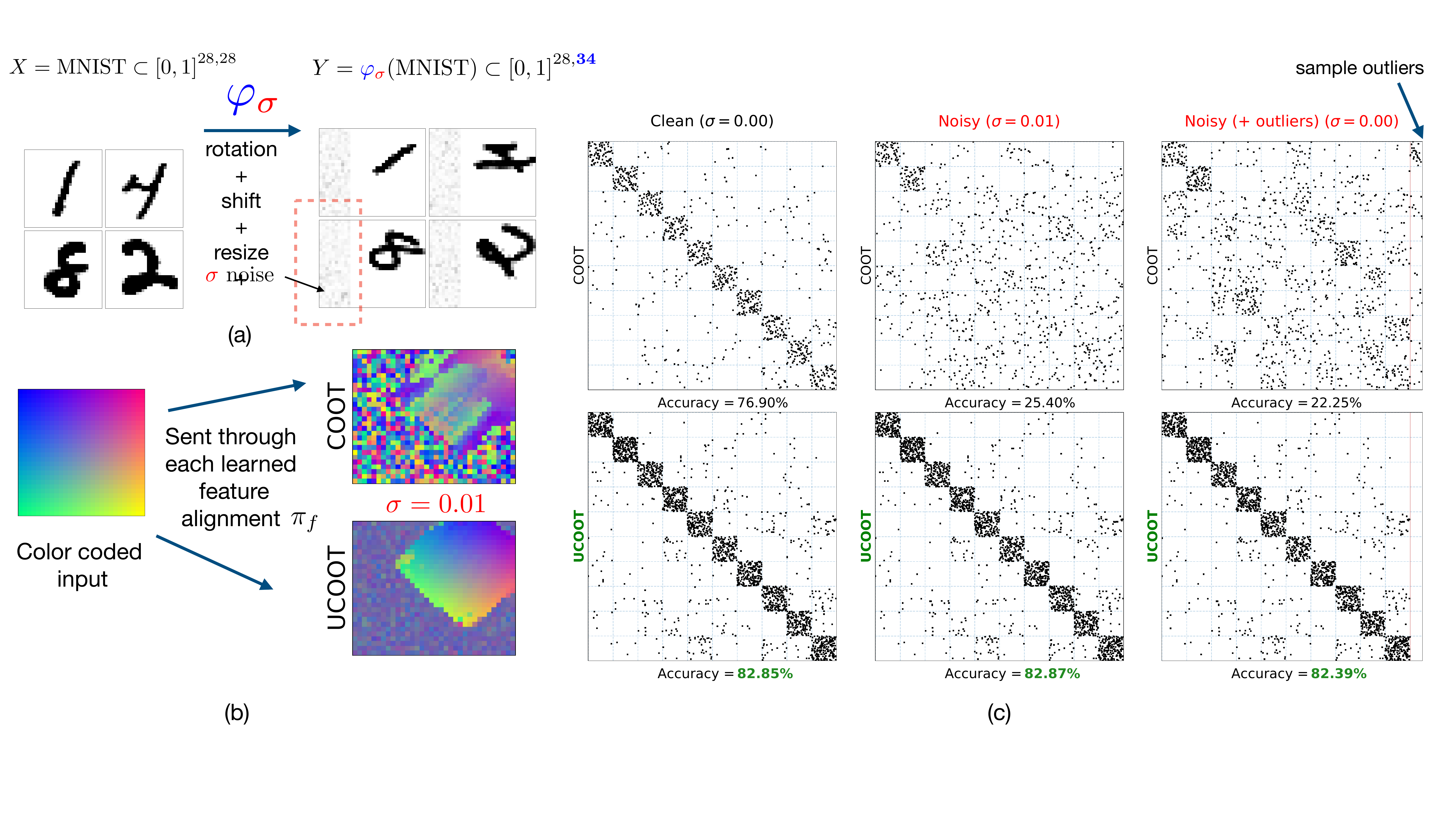}
  \caption{Example illustrating the feature alignment $\pi_f$ learned by UCOOT and its robustness to outliers.
  \textbf{(a)} Visualization of 4 random samples from both datasets. The added Gaussian noise only affects the first 10 columns of the images and is different across images. \textbf{(b)} The barycentric mapping (see Appendix for details) defined by UCOOT learns the transformation defined by $\varphi_\sigma$ while disregarding non-informative features. \textbf{(c)} Alignments across samples from $X$ and $Y$. We contaminated the target $Y$ with 50 sample outliers (images with uniform entries in $[0,1]$). A very small amount of noise is sufficient to derail COOT. Unlike COOT, UCOOT does not transport any outlier sample. Accuracy is computed as the percentage of mass within the block-diagonal structure.
  \label{f:mnist-example}}
\end{figure*}

\subsection{Heterogeneous Domain Adaptation (HDA)}

%Domain adaptation (DA) refers to the problem in which a classifier 
%learned on one domain (called \textit{source}) can generalise well to the other one (called \textit{target}). 
We now investigate the application of discrete UCOOT in unsupervised Heterogeneous Domain Adaptation (HDA). It is a particularly difficult problem where one wants to predict classes on unlabeled data using labeled data lying in a different space. OT methods across spaces have recently shown good performance on such tasks, in particular using GW distance ~\citep{Yan18} and COOT~\citep{Redko20}.
%i.e. the samples in source and target domains live in the different spaces, and we have no access to the labelled target sample. 
%It should be noted that COOT and UCOOT have already been used in real-world applications. 
% For example, as the entropic GW distance and its COOT follow exactly the same approximation scheme, they share the success in 
% graph matching \citep{Xu19b}, correspondance alignment \citep{Solomon16}, 
% comparing kernel matrices \citep{peyre16}. On the other hand, the recent discrete COOT also works well in co-clustering and 
% HDA tasks \citep{Redko20} and the UCOOT has shown competitive performance in positive-unlabelled learning \citep{Sejourne20}.

\paragraph{Datasets and experimental setup.} We consider the Caltech-Office dataset~\citep{Saenko10} containing three domains: 
Amazon (A) ($1123$ images), Caltech-$256$ (C) ($958$ images) and Webcam (W) ($295$ images) 
with 10 overlapping classes amongst them. The image in each domain is representated by the output of the second last layer 
in the Google Net~\citep{Szegedy15} and Caffe Net~\citep{Jia14} neural network architectures, which results in 
$4096$ and $1024$-dimensional vectors, respectively (thus $d_s = 4096, d_t = 1024$). 
%\paragraph{Competing methods} 
We compare $4$ OT-based methods: GW, COOT, UGW, and UCOOT. The hyper-parameters for each method are validated on a unique pair of datasets (W$\rightarrow$W), then fixed for all other pairs in order to provide truly unsupervised HDA generalization. We follow the same experimental setup as in~\citep{Redko20}. For each pair of domains, we randomly choose $20$ samples per class (thus $n_s = n_t = 200$) and 
perform adaptation from CaffeNet to GoogleNet features, then calculate the accuracy of the generated predictions on the target domain using OT label propagation \citep{Redko19a}. 
\begin{wrapfigure}[12]{r}{0.25\textwidth}
% \begin{figure}
%   \centering
  \includegraphics[clip, width=1.05\linewidth]{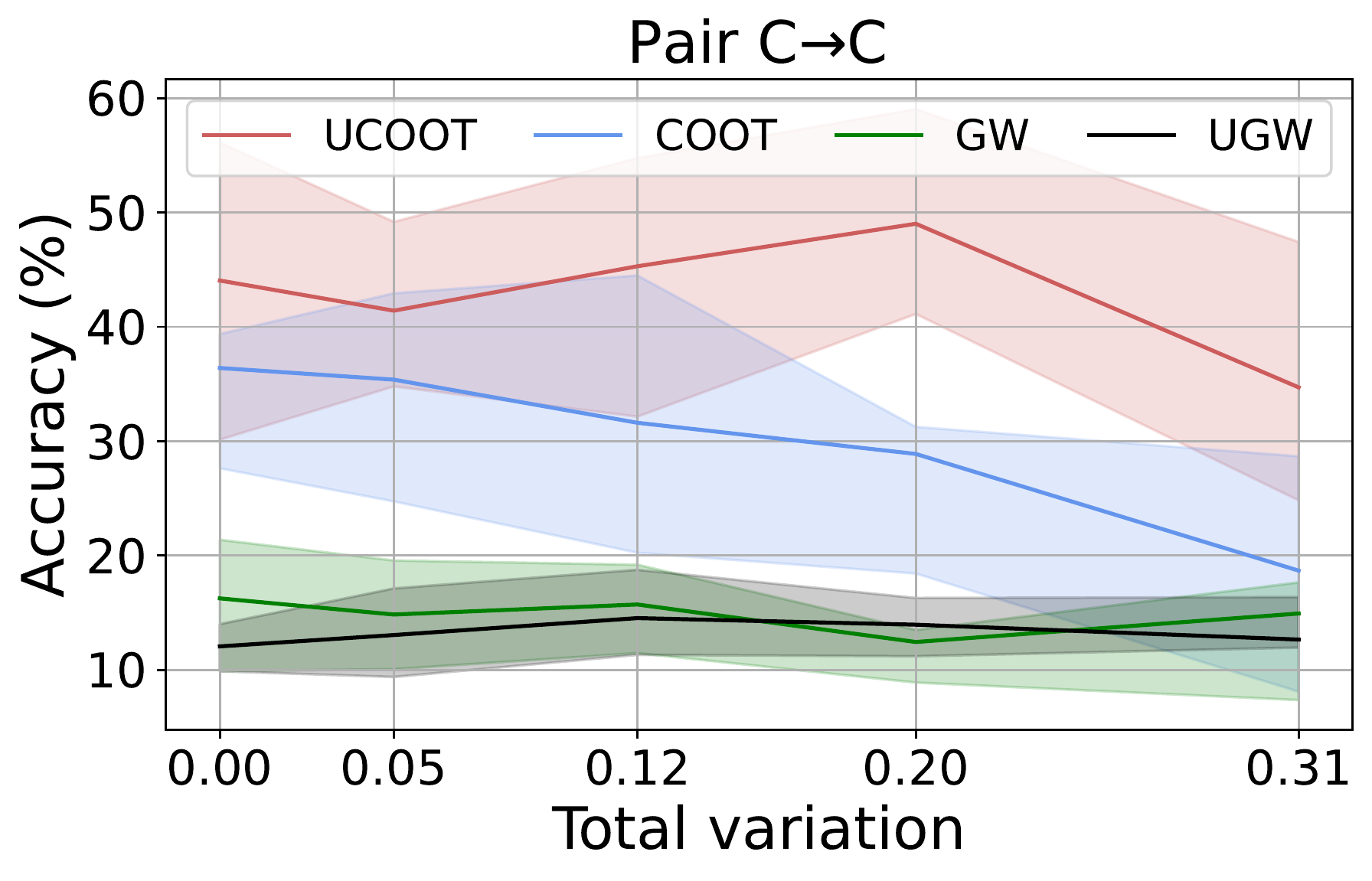}
  \caption{Robustness to class proportion change for increasing TV on the class marginals.\label{f:hda_prop}}
% \end{figure}
\end{wrapfigure}
This technique uses the OT plan to estimate the amount of mass transported from each class (since the sources are labeled) to a given target sample. The predicted class corresponds to the one which contains the most mass.
% : except for the baseline, for each method, we use the learned sample coupling between source and 
% target data to predict the labels in the target domain via label propagation \citep{Redko19a}.
% More precisely, for each pair of domains, we randomly choose $20$ samples per class (thus $n_s = n_t = 200$) and 
% perform adaptation from CaffeNet to GoogleNet features, then calculate the accuracy of the generated predictions on the target domain. 
We repeat this process $10$ times and calculate the average and standard deviation of the performance.
In both source and target domains, we assign uniform sample and feature distributions. 

% In the semi-supervised setting, we incorporate the prior knowledge on the target labels by adding an 
% additional cost matrix to the training of sample coupling, so that a source sample will be penalised if it transfers mass to the 
% target samples in the different classes. More precisely, we introduce the masked target label $\tilde{y}^{(t)} \in \mathbb R^{n_t}$ 
% defined by randomly keeping $\tilde{n}_t \in \{1,3,5\}$ samples in each class in the target label $y^{(t)}$ and masking all other 
% labels in $y^{(t)}$ by $-1$. Then the additional cost $M \in \mathbb R^{n_s \times n_t}$ between $y^{(s)}$ and $\tilde{y}^{(t)}$ is defined by
% \begin{equation}
% M_{ij} = 
% \begin{cases}
% 0, \text{ if } y^{(s)}_i = \tilde{y}^{(t)}_j, \text{ or } \tilde{y}^{(t)}_j = -1 \\
% v, \text{ otherwise}.
% \end{cases}
% \end{equation}
% Here, $v > 0$ is a fixed value and we choose $v = 100$ in this experiment.

% Note that, while the prediction is performed on the whole target samples, only those whose labels are masked as $-1$ during the 
% training, are used in the calculation of accuracy. For the $k$-NN only, we train a classifier on the labelled target samples, 
% then perform prediction on the unlabelled ones.

\begin{table}[t]
  \centering
  \scriptsize
		\begin{tabular}{c c c c c}
				\toprule
				& \multicolumn{3}{c}{CaffeNet $\to$ GoogleNet} \\
				\midrule
				Domains & GW & UGW & COOT & UCOOT \\
				\midrule
				
				C $\to$ C & 16.25 ($\pm$ 7.54) & 10.85 ($\pm$ 2.13) & 36.40 ($\pm$ 12.94) & \textbf{44.05 ($\pm$ 19.33)} \\
				\hline
				C $\to$ A & 12.95 ($\pm$ 7.74) & 11.60 ($\pm$ 4.86) & 28.30 ($\pm$ 11.78) & \textbf{31.90 ($\pm$ 7.43)} \\
				\hline
				C $\to$ W & 18.95 ($\pm$ 9.43) & 14.15 ($\pm$ 3.98) & 19.55 ($\pm$ 14.51) & \textbf{28.55 ($\pm$ 6.60)} \\
				\hline
				
				A $\to$ C & 16.40 ($\pm$ 8.99) & 10.25 ($\pm$ 5.66) & \textbf{41.80 ($\pm$ 14.81)} & 39.15 ($\pm$ 17.98) \\
				\hline
				A $\to$ A & 14.75 ($\pm$ 15.20) & 20.20 ($\pm$ 6.45) & \textbf{57.90 ($\pm$ 16.84)} & 42.45 ($\pm$ 15.47) \\
				\hline
				A $\to$ W & 14.55 ($\pm$ 8.83) & 20.65 ($\pm$ 4.13) & 42.10 ($\pm$ 7.80) & \textbf{48.55 ($\pm$ 13.06)} \\
				\hline
				
				W $\to$ C & 20.65 ($\pm$ 11.90) & 14.20 ($\pm$ 5.13) & 8.60 ($\pm$ 6.56) & \textbf{69.80 ($\pm$ 14.91)} \\
				\hline
				W $\to$ A & 17.00 ($\pm$ 9.75) & 7.10 ($\pm$ 2.45) & 16.65 ($\pm$ 10.01) & \textbf{30.55 ($\pm$ 10.09)} \\
				\hline
				W $\to$ W & 19.30 ($\pm$ 11.87) & 24.40 ($\pm$ 3.28) & \textbf{75.30 ($\pm$ 3.26)} & 51.50 ($\pm$ 20.51) \\
				\bottomrule
				Average & 16.76 ($\pm$ 10.14) & 14.82 ($\pm$ 4.23) & 36.29 ($\pm$ 10.95) & \textbf{42.94 ($\pm$ 13.93)} \\    
				\bottomrule
			\end{tabular}
	\caption{Unsupervised HDA from CaffeNet to GoogleNet. \label{tab:hda}}
\end{table}

\paragraph{HDA Results and robustness to target shift.} The means and standard deviations of the accuracy on target data are reported in Table \ref{tab:hda} for all the methods and all pairs of datasets. We observe that, thanks to its robustness, UCOOT outperforms COOT on 7 out of 9 dataset pairs, with higher average accuracy but also slightly larger variance. This is because of the difficulty of the unsupervised HDA problem and the instability present in all methods. In particular, GW-based approaches perform very poorly. This may be due to the fact that the pre-trained models contain meaningful but a very high-dimensional vectorial representation of the image. Thus, using the Euclidean distance matrices as inputs not only causes information loss but also is less relevant (see for example, \citep{Aggarwal01}, or Theorem 3.1.1 and Remark 3.1.2 in \citep{Vershynin18}). We also illustrate the robustness of UCOOT to a change in class proportions, also known as target shift. More precisely, we simulate a change in proportion only in the source domain by selecting $20\rho$ samples per class for 4 amongst 10 classes with $\rho$ decreasing from $\rho=1$ to $\rho=0.2$. In this configuration, the classes in the source domain are imbalanced and the unlabeled HDA problem becomes more difficult. We report the performance of all the methods as a function of the Total Variation (TV) between the class marginal distributions on one pair of datasets in Figure~\ref{f:hda_prop}. We can see that UCOOT is quite robust to change in class proportions, while COOT experiences a sharp decrease in accuracy when the class distributions become more imbalanced.

% We observe that in general, UCOOT is less stable than COOT, due to the mass 
% relaxation nature. This is advantageous in the unsupervised HDA tasks, where, in absence of 
% target labels, this relaxation provides flexibility to deals with the uncertainty of labels. 
% For this reason, UCOOT usually outperforms COOT by large margins.
%
%\paragraph{Robustness to target shift}
\subsection{Single-cell multi-omics alignment}
Finally, we present a real-world application of UCOOT for the alignment of single-cell measurements. Recent advances in single-cell sequencing technologies allow biologists to measure a variety of cellular features at the single-cell resolution, such as expression levels of genes and epigenomic changes in the genome \cite{Buenrostro15,Chen19}, or the abundance of surface proteins \cite{CITEseq}. These multiple measurements produce single-cell multi-omics datasets. These datasets measuring different biological phenomena at the single-cell resolution allow scientists to study how the cellular processes are regulated, leading to finer cell variations during development and diseases. However, it is hard to obtain multiple types of measurements from the same individual cells due to experimental limitations. Therefore, many single-cell multi-omics datasets have disparate measurements from different sets of cells. As a result, computational methods are required to align the cells and the features of the different measurements to learn the relationships between them that help with data analysis and integration. Multiple tools \cite{Pamona, Seurat, Liu2019}, including GW \cite{Pamona, Demetci22} and UGW \cite{Demetci22-2} based methods, have shown good performance for cell-to-cell alignments.
% However,  aligning both samples and features is a more challenging and critical task for this domain that can be solved by the UCOOT formulation. 
However, aligning both samples and features is a more challenging and critical task that GW and UGW-based methods cannot address 
%\pinar{since they compute alignments between intra-domain distances, discarding features \cite{Demetci22, Pamona, Demetci22-2}}. 
Here we provide an application of UCOOT to simultaneously align the samples and features in a single-cell multi-omics dataset.
%
%how the features in different molecular features relate to one another.
%As a result, we require computational methods to align the cells and the features of the different measurements to obtain correspondence between similar cells and features. This correspondence information can help investigate which features in one domain (e.g., genes) match the ones (e.g., proteins) in the other domain.
%
\begin{figure*}[t]
    \centering
    \includegraphics[trim={0.2cm 0.2cm 0.8cm 0.5cm}, clip, width=0.95\linewidth, keepaspectratio=true]{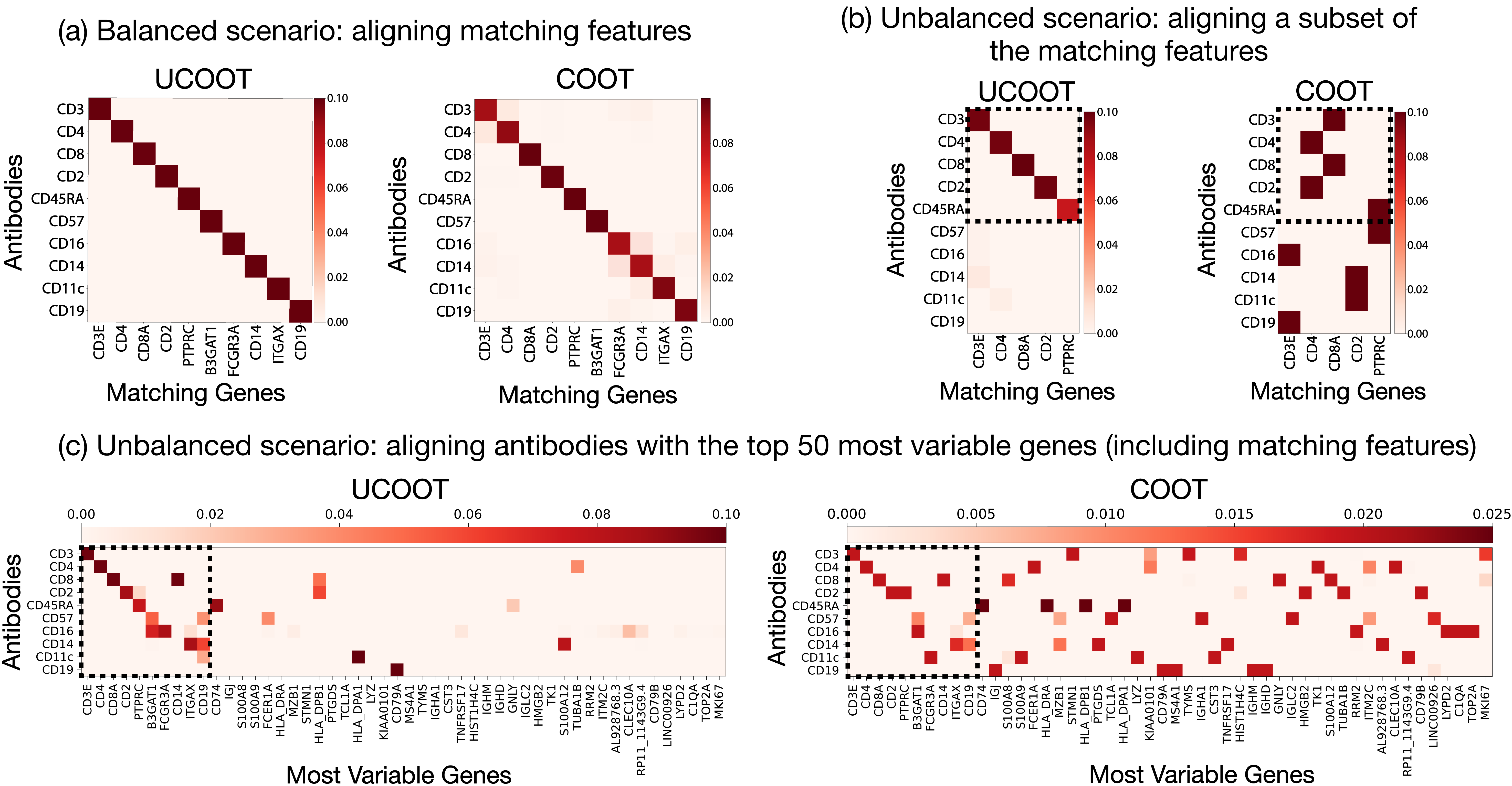}
    \caption{Feature alignments on the single-cell multi-omics dataset of COOT and UCOOT between antibodies (surface proteins) and their matching genes (that encode them). \textbf{(a)} The features are sorted such that the correct alignment would yield a diagonal matrix. \textbf{(b)} Only five of the correct gene matches are kept (the last five genes from (a) are excluded). \textbf{(c)} Alignments between the ten antibodies and the top 50 most variable genes, including the matching genes.
    For \textbf{(b)} and \textbf{(c)}, the diagonal within the dashed square highlights the correct matches.
    Overall, UCOOT gives better feature alignments.
    \label{fig:multiomics}}
\end{figure*}

For demonstration, we choose a dataset generated by the CITE-seq experiment \cite{CITEseq}, which simultaneously measures gene expression and antibody (or surface protein) abundance in single cells. From this dataset, we use 1000 human peripheral blood cells, which have ten antibodies and 17,014 genes profiled. We selected this specific dataset as we know the ground-truth correspondences on both the samples (i.e., cells) and the features (i.e., genes and their encoded antibodies), thus allowing us to quantify and compare the alignment performance of UCOOT and COOT. As done previously \cite{Pamona, Demetci22, Liu2019}, we quantify the cell alignment performance by calculating the fraction of cells closer than the true match (FOSCTTM) of each cell in the dataset and averaging it across all cells. This metric quantifies alignment error, so lower values are more desirable. The feature alignments are measured by calculating the accuracy of correct matching. The results are presented after hyperparameter tuning both methods with similar grid size per hyperparameter (see Experimental Details in Appendix).

% I'll move the details about the original dataset to the Appendix
% The particular dataset we work with profiles a mix of 7,985 human and mouse peripheral blood mononuclear cells and contains measurements on 10 antibodies, 17014 human genes and 12915 mouse genes. We select 1000 human cells from this dataset to work with.

\textbf{Balanced Scenario.} First, we select and align the same number of samples and features across the two datasets. For this, we subset the gene expression domain with the ten genes that match to the ten antibodies they express. Original data contains the same number of cells across domains since both domains are simultaneously measured in the same single-cells. We observe that both UCOOT and COOT can correctly align features (Figure~\ref{fig:multiomics}~(a)) and the cells (Appendix Figure S1(a)) across the two measurements. However, UCOOT gives better performance, as demonstrated by a lower FOCSTTM score (0.0062 vs 0.0127) for cells. Both COOT and UCOOT recover the diagonal for matching features (100\% accuracy), but UCOOT recovers the exact matching, likely due to its robustness to noise, whereas COOT assigns weights to other features as well. 

\textbf{Unbalanced Scenarios.} Next, we perform alignment with an unequal number of features. This setting is more likely to occur for real-world single-cell datasets as different features are measured. In the first simple scenario, we align the ten antibodies with only a subset (five) of their matching genes. As visualized in Figure~\ref{fig:multiomics}~(b), COOT struggles to find the correct feature alignments (60\% accuracy), which would lie in the diagonal of the highlighted square (dashed lines). However, the relaxation of the mass conservation constraint in UCOOT allows it to shrink the mass of antibodies that lack matches in the gene expression domain, leading do higher accuracy (100\% accuracy). Next, we align the ten antibodies with the 50 most variable genes in the dataset, including their matching genes. This alignment task is the most realistic scenario, as single-cell multi-omics data consists of high-dimensional datasets with a different number of features for different measurements. 
% \pinar{Moreover, many of these genes show consistent expression levels across } 
Therefore, biologists focus their analyses on the reduced set of most variable features (e.g. genes). It is also the most computationally challenging case among all our experiments on this dataset. Hence, we provide sample-level supervision to both methods by giving a cost penalization matrix based on the correct sample alignments to the sample alignment computation. We see in Figure ~\ref{fig:multiomics}(c) that in comparison to COOT (50\% accuracy), UCOOT recovers more of the correct feature alignments (70\% accuracy), and yields fewer redundant alignments  (for more detail, see Experimental Details in Appendix). Note that UCOOT avoids incorrect mappings by locally shrinking the mass of the features or samples that lack correspondences. This avoids subsequent incorrect downstream analysis of the integration results. This property can also help users to discover rare cell types by observing the extent of mass relaxation per cell or prune uninformative features in the single-cell datasets.

%Similarly to the alignment in Figure ~\ref{fig:multiomics}(b), this is thanks to the mass shrinkage of the features that lack correspondences. By locally adjusting the mass in transport, UCOOT tends to avoid assigning correspondences to samples or features that lack matches across the aligned domains. With this, UCOOT can help users to detect outliers that can result from experimental artifacts, 
%(such as rare cell types or lowly expressed genes that are not captured well by a sequencing technology), prune features or discover rare cell types.}

Lastly, we also consider the case of unequal number of samples across the two measurements. This case is common in real world single-cell multi-omics datasets that are not simultaneously measured. Demetci \textit{et al.} \cite{Demetci22-2} have shown that single-cell alignment methods that do not account for this mismatch yield poor alignment results. Therefore, we downsample the number of cells in one of the domains by 25\% and perform alignment with the full set of cells in the other domain (Appendix Figure S1(b \& c). We compute the FOSCTTM score for all cells that have a true match in the dataset and report the average values. UCOOT continues to yield a low FOSCTTM score (0.0081 compared to 0.0062 in the balanced scenario), while COOT shows a larger drop in performance (0.1342 compared to 0.0127 in the balanced scenario).

%%%%%%%%%%%%%%%%%%%%%%%%%%%%%%%%%%%%%%%%%

% \begin{ack}
% Use unnumbered first level headings for the acknowledgments. All acknowledgments
% go at the end of the paper before the list of references. Moreover, you are required to declare
% funding (financial activities supporting the submitted work) and competing interests (related financial activities outside the submitted work).
% More information about this disclosure can be found at: \url{https://neurips.cc/Conferences/2022/PaperInformation/FundingDisclosure}.

% \end{ack}

\section{Discussion and conclusion}
\label{sec:conclusion}

In this work, we present an extension of COOT called unbalanced COOT, where the hard constraint on the marginal distributions is replaced by a soft control via the KL divergence. The resulting problem not only benefits from the flexibility of COOT but also enjoys the provable robustness property under the presence of outliers, which is not the case for COOT. The experimental results confirm our findings, yielding a very competitive performance in the unsupervised HDA task, as well as meaningful feature couplings for the single-cell multi-omics alignment. Also, while UCOOT introduces additional hyper-parameters, domain knowledge can help narrow down the range of feasible values, thus reducing the time and computational cost of the tuning process. Further investigation should be carried out to fully understand and assess the observed efficiency of UCOOT in real-world applications,
and also explore the possibilities of UCOOT in more diverse applicative settings, including its use as a loss in deep learning architectures. 
Lastly, from a theoretical perspective, statistical properties such as sample complexity or stability analysis are needed to better understand the intricate relations between the two sample and feature couplings.  
%Future works may also focus on justifying how the relaxation of marginal constraints allows to cope with the class imbalance and provides empirically better initialization to the approximation scheme of COOT, as observed in the experiments. Last but not least, it is interesting to study the connection with the unbalanced GW distance, since in practice, it is estimated using its lower bound, which is in fact a special of UCOOT.

\section*{Acknowledgement}

The authors thank to Tuan Binh Nguyen for the fruitful discussion and invaluable suggestions. This work is funded by the projects OTTOPIA ANR-20-CHIA-0030, 3IA Côte d'Azur Investments ANR-19-P3IA-0002 of the French National Research Agency (ANR), the 3rd Programme d’Investissements d’Avenir ANR-18-EUR-0006-02, the Chair "Challenging Technology for Responsible Energy" led by l’X – Ecole Polytechnique and the Fondation de l’Ecole Polytechnique, sponsored by TOTAL, and the Chair "Business Analytics for Future Banking" sponsored by NATIXIS. This research is produced within the framework of Energy4Climate Interdisciplinary Center (E4C) of IP Paris and Ecole des Ponts ParisTech. Pinar Demetci's and Ritambhara Singh's contribution was funded by R35~HG011939.

\bibliography{aaai23}

\onecolumn

\appendix

% !TEX root = ../main.tex
\renewcommand{\thesection}{S\arabic{section}}
\renewcommand{\thefigure}{S\arabic{figure}}
\renewcommand{\thetable}{S\arabic{table}}
\setcounter{section}{0}
\setcounter{figure}{0}
\setcounter{table}{0}

\appendix
\setlength{\tabcolsep}{0.3em}
{\renewcommand{\arraystretch}{1}}

\section{Appendix}

% Optionally include extra information (complete proofs, additional experiments and plots) in the appendix.
% This section will often be part of the supplemental material.

%%%%%%%%%%%%%%%%%%%%%%%%%%%%%%%%%%%%%%
\section{Additional concepts and notations}
Denote $\mathcal C_b(\mathcal X)$ the space of bounded continuous functions on $\mathcal X$. 
Given a Borel measurable map $T: \mathcal X \to \mathcal Y$ and a measure $\mu \in \mathcal{M}^+(\mathcal X)$, 
we define the push-forward (or image) measure $T_{\#}\mu \in \mathcal{M}^+(\mathcal Y)$ is the one which satisfies: 
for every $\phi \in \mathcal C_b(\mathcal Y), \int_{\mathcal X} (\phi \circ T) \; d\mu = \int_{\mathcal Y} \phi \; dT_{\#}\mu$.
For example, the marginal distributions of a measure on $\mathcal X \times \mathcal Y$ are the push-forward measures induced by the canonical projections:
$P_{\mathcal X}(x,y) = x$ and $P_{\mathcal Y}(x,y) = y$. Given two probability measures 
$(\mu, \nu) \in \mathcal{M}^+_1(\mathcal X) \times \mathcal{M}^+_1(\mathcal Y)$, 
denote $U(\mu, \nu) = \{ \pi \in \mathcal M^+_1(\mathcal X \times \mathcal Y): 
\pi_{\#1} = \mu, \pi_{\#2} = \nu\}$ the set of admissible transport plans.

Given an entropy function $\varphi : \mathbb R_{> 0} \to [0, \infty]$ 
(i.e. it is convex, positive and lower semi-continuous such that $\varphi(1) = 0$), we define the recession constant $\varphi_{\infty}' \in \mathbb R \cup \{ \infty \}$ as
$\varphi_{\infty}' = \lim_{x \to \infty} \frac{\varphi(x)}{x}$. The Csiszár divergence (or $\varphi$-divergence) \cite{Csiszar63} between two 
measures $\mu$ and $\nu$ in a certain space $\mathcal M^+(\mathcal S)$ is defined as
$D_{\varphi}(\mu \vert \nu) = \int_\mathcal S \varphi \Big( \frac{d\mu}{d\nu} \Big) d\nu + 
\varphi_{\infty}' \int_\mathcal S d\mu^{\perp}$,
where, by Lebesgue decomposition, we have $\mu = \frac{d\mu}{d\nu} \nu + \mu^{\perp}$. For example,
\begin{itemize}
	\item If $\varphi(x) = x \log x -x + 1$, then $D_{\varphi}$ is the Kullback-Leibler (KL) divergence.

	\item If $\varphi(x)$ is equal to $0$ if $x=1$ and $+\infty$ otherwise, then $D_{\varphi}$ is the indicator divergence $\iota_{=}$.
\end{itemize}
%%%%%%%%%%%%%%%%%%%%%%%%%%%%%%%%%%%%%%
For later convenience, we define the function $\vert \xi_1 - \xi_2 \vert^p: (\mathcal X_1^s \times \mathcal X_2^s) \times (\mathcal X_1^f \times \mathcal X_2^f) \to \mathbb R_{\geq 0}$ by
\begin{equation*}
    \vert \xi_1 - \xi_2 \vert^p \big((x_1^s, x_2^s), (x_1^f, x_2^f)\big) :=  \vert \xi_1(x_1^s, x_1^f) - \xi_2(x_2^s, x_2^f) \vert^p,
\end{equation*}
and write the objective function of generalised COOT as
\begin{equation*}
    F_{\lambda}(\pi^s, \pi^f) = \iint |\xi_1 - \xi_2|^p \mathrm d\pi^s \mathrm d \pi^f + \sum_{k=1}^2\lambda_k D_k(\pi^s_{\#k} \otimes \pi^f_{\#k} \vert \mu^s_k \otimes \mu^f_k).
\end{equation*}
The generalized COOT now reads compactly as
\begin{equation} \label{eq:ucoot_copy}
%   \ucoot_{\lambda}(\mathbb X_1, \mathbb X_2) := 
  \inf_{\substack{\pi^s \in \mathcal M^+(\mathcal X_1^s \times \mathcal X_2^s) \\
  \pi^f \in \mathcal M^+(\mathcal X_1^f \times \mathcal X_2^f) \\ m(\pi^s) = m(\pi^f)}} F_{\lambda}(\pi^s, \pi^f)
\end{equation}
%%%%%%%%%%%%%%%%%%%%%%%%%%%%%%%%%%%%
\section{Proofs related to UCOOT}
%%%%%%%%%%%%%%%%%%%%%%%%%%%%%%%%%%%%%%%%%%

\subsection{UCOOT and its properties}
%%%%%%%%%%%%%%%%%%%%%%%%%%%%%%%%%%%%%%%%%%
\begin{claim}
  When $D_k = \iota_{=}$ and $\mu_k^s, \mu_k^f$ are probability measures, for $k=1,2$, then we recover COOT from generalized COOT.
\end{claim}
%%%%%%%%%%%%%%%%%%%%%%%%%%%%%%%%%%%%%%%%%%
\begin{proof}
  Under the above assumptions, the generalized COOT problem becomes
  \begin{equation*}
    \begin{split}
      \inf_{\substack{\pi^s \in \mathcal M^+(\mathcal X_1^s \times \mathcal X_2^s) \\ 
      \pi^f \in \mathcal M^+(\mathcal X_1^f \times \mathcal X_2^f)}} 
      &\iint |\xi_1 - \xi_2|^p \mathrm d\pi^s \mathrm d \pi^f \\
      \text{subject to } &\pi^s_{\#1} \otimes \pi^f_{\#1} = \mu_1^s \otimes \mu_1^f \text{ (C1) } \\
      & \pi^s_{\#2} \otimes \pi^f_{\#2} = \mu_2^s \otimes \mu_2^f \text{ (C2) } \\
      & m(\pi^s) = m(\pi^f) \text{ (C3) }.
    \end{split}
  \end{equation*}
  As $m(\pi) = m(\pi_{\#1}) = m(\pi_{\#2})$,
  for any measure $\pi$, and $\mu_k^s, \mu_k^f$ are probability measures, for $k=1,2$, 
  one has $m(\pi^s) m(\pi^f) = 1$, thus $m(\pi^s) = m(\pi^f) = 1$.
  Now, the constraint C1 implies that $\int_{\mathcal X_1^s} \mathrm d\pi^s_{\#1} \mathrm d \pi^f_{\#1} = \int_{\mathcal X_1^s} \mathrm d\mu_1^s \mathrm d\mu_1^f$. Thus, $\pi^f_{\#1} = \mu_1^f$. Similarly, we have $\pi^s_{\#k} = \mu_k^s$ and $\pi^f_{\#k} = \mu_k^f$, for any $k=1,2$. We conclude that $\pi^f \in U(\mu_1^f, \mu_2^f)$ and $\pi^s \in U(\mu_1^s, \mu_2^s)$, 
  and we obtain the COOT problem. \qed
\end{proof}
%%%%%%%%%%%%%%%%%%%%%%%%%%%%%%%%%%%%%%%%%%

%%%%%%%%%%%%%%%%%%%%%%%%%%%%%%%%%%%%%%%%%%
\begin{proposition}
    \label{eq:ucoot_existence_copy}
  (Existence of minimizer) Denote $\mathcal S := (\mathcal X_1^s \times \mathcal X_2^s) \times (\mathcal X_1^f \times \mathcal X_2^f)$. The problem \ref{eq:ucoot_copy} admits a minimizer if at least one of the following conditions hold:
  \begin{enumerate}
    \item The entropy functions $\phi_1$ and $\phi_2$ are superlinear, i.e. $(\phi_1)'_{\infty} = (\phi_2)'_{\infty} = \infty$.
    \item The function $\vert c_X - c_Y \vert^p$ has compact sublevels in $\mathcal S$ and 
    $\inf_{\mathcal S} \vert \xi_1 - \xi_2 \vert^p + \lambda_1 (\phi_1)'_{\infty} + \lambda_2 (\phi_2)'_{\infty} > 0$.
  \end{enumerate}
\end{proposition}
%%%%%%%%%%%%%%%%%%%%%%%%%%%%%%%%%%%%%%%%%%
%%%%%%%%%%%%%%%%%%%%%%%%%%%%%%%%%%%%%%%%%%
\begin{proof}
  We adapt the proof of Theorem 3.3 in \citep{Liero18} and of Proposition 3 in \citep{Sejourne20}. 
  For convenience, we write $\mu_1 = \mu_1^s \otimes \mu_1^f$ and 
  $\mu_2 = \mu_2^s \otimes \mu_2^f$. For each pair $(\pi^s, \pi^f)$, denote 
  $\pi = \pi^s \otimes \pi^f$. 
  It can be shown that 
  $\pi_{\# k} := (P_{\mathcal X_k^s \times \mathcal X_k^f})_{\#} \pi 
  = (P_{\mathcal X_k^s})_{\#} \pi^s \otimes (P_{\mathcal X_k^f})_{\#} \pi^f = 
  \pi^s_{\# k} \otimes \pi^f_{\# k}$, for $k=1,2$. Indeed, for any function $\phi \in \mathcal C_b(\mathcal X_k^s \times \mathcal X_k^f)$, we have
    \begin{equation*}
      \begin{split}
        \int_{\mathcal X_k^s \times \mathcal X_k^f} \phi \; \mathrm d (P_{\mathcal X_k^s \mathcal X_k^f})_{\#} \pi 
        &= \int_{\mathcal S} (\phi \circ P_{\mathcal X_k^s \mathcal X_k^f}) \; \mathrm d\pi \\
        &= \int_{\mathcal S} \phi(x_k^s, x_k^f) 
        \; \mathrm d \pi^s(x_1^s, x_2^s) \; \mathrm d \pi^f(x_1^f, x_2^f) \\  
        &= \int_{\mathcal X_k^s \times \mathcal X_k^f} \phi \; \mathrm d \pi^s_{\# k} \; 
        \mathrm d \pi^f_{\# k}.
      \end{split}
    \end{equation*}
  Thus, the problem \ref{eq:ucoot_copy} can be rewritten as
  \begin{equation*}
    \ucoot_{\lambda}(\mathbb X_1, \mathbb X_2) = 
    \inf_{\pi \in E_{uco}} \int_{\mathcal S} \vert \xi_1 - \xi_2 \vert^p 
    \mathrm d\pi + \sum_{k=1,2} \lambda_k D_{\phi_k}(\pi_{\# k} \vert \mu_k),
  \end{equation*}
  where
  \begin{equation*}
    E_{uco} = \{ \pi \in \mathcal M^+(\mathcal S) \vert \pi = \pi^s \otimes \pi^f, 
    \pi^s \in \mathcal M^+(\mathcal X_1^s \times \mathcal X_2^s), 
    \pi^f \in \mathcal M^+(\mathcal X_1^f \times \mathcal X_2^f) \}.
  \end{equation*}
  Define 
  \begin{equation*}
    L(\pi):= \int_{\mathcal S} \vert \xi_1 - \xi_2 \vert^p \mathrm d \pi + 
    \sum_{k=1,2} \lambda_k D_{\phi_k}(\pi_{\# k} \vert \mu_k).
  \end{equation*}
  By Jensen's inequality, we have
  \begin{equation*}
    \begin{split}
      L(\pi) &\geq m(\pi) \inf_{\mathcal S} \vert \xi_1 - \xi_2 \vert^p + 
      \sum_{k=1,2} \lambda_k m(\mu_k) \phi_k \Big( \frac{m(\pi_{\# k})}{m(\mu_k)} \Big) \\
      &= m(\pi) \bigg[ \inf_{\mathcal S} \vert \xi_1 - \xi_2 \vert^p + 
      \sum_{k=1,2} \lambda_k \frac{m(\mu_k)}{m(\pi)} \phi_k \Big( \frac{m(\pi)}{m(\mu_k)} \Big) \bigg],
    \end{split}
  \end{equation*}
  where, in the last equality, we use the relation $m(\pi) = m(\pi_{\# k})$, for $k=1,2$. 
  It follows from the assumption that $L$ is coercive, i.e. $L(\pi) \to \infty$ when $m(\pi) \to \infty$.
  
  Clearly $\inf_{E_{uco}} L < \infty$ because $L\big( (\mu_1^s \otimes \mu_2^s) \otimes (\mu_1^f \otimes \mu_2^f) \big) < \infty$.
  Let $(\pi_n)_n \subset {E_{uco}}$ be a minimizing sequence, i.e. $L(\pi_n) \to \inf_{E_{uco}} L$. 
  Such sequence is necessarily bounded (otherwise, there exists a subsequence $(\pi_{n_k})_{n_k}$ 
  with $m(\pi_{n_k}) \to \infty$ and the coercivity of $L$ implies $L(\pi_{n_k}) \to \infty$, 
  which is absurd). Suppose $m(\pi_{n}) \leq M$, for some $M > 0$. By Tychonoff's theorem, 
  as $\mathcal X_k^s$ and $\mathcal X_k^f$ are compact spaces, 
  so is the product space $\mathcal S$. Thus, by Banach-Alaoglu theorem, 
  the ball $B_M = \{ \pi \in \mathcal M^+(\mathcal S): m(\pi) \leq M \}$ 
  is weakly compact in $\mathcal M^+(\mathcal S)$. 
  
  Consider the set $\overline{E}_{uco} = E_{uco} \cap B_M$, then clearly $(\pi_n)_n \subset \overline{E}_{uco}$. We will show that 
  there exists a converging subsequence of $(\pi_n)_n$, whose limit is in $\overline{E}_{uco}$, 
  thus $\overline{E}_{uco}$ is weakly compact. Indeed, by definition of $E_{uco}$, 
  there exist two sequences $(\pi_n^s)_n$ and $(\pi_n^f)_n$ such that 
  $\pi_n = \pi_n^s \otimes \pi_n^f$. 
  We can assume furthermore that $m(\pi_n^s) = m(\pi_n^f) = \sqrt{m(\pi_n)} \leq \sqrt M$. 
  As $m(\pi_n^s)$ and $m(\pi_n^f)$ are bounded, by reapplying Banach-Alaoglu theorem, 
  one can extract two converging subsequences (after reindexing) 
  $\pi_n^s \rightharpoonup \pi^s \in \mathcal M^+(\mathcal X_1^s \times \mathcal X_2^s)$ and 
  $\pi_n^f \rightharpoonup \pi^f \in \mathcal M^+(\mathcal X_1^f \times \mathcal X_2^f)$, 
  with $m(\pi^s) = m(\pi^f) \leq \sqrt{M}$. 
  An immediate extension of Theorem 2.8 in \citep{Billingsley99} to the convergence of the products of bounded positive measures implies 
  $\pi_n^s \otimes \pi_n^f \rightharpoonup \pi^s \otimes \pi^f \in \overline{E}_{uco}$. 
  
  Now, the lower semicontinuity of $L$ implies that $\inf_{E_{uco}} L \geq L(\pi^s \otimes \pi^f)$, 
  thus $L(\pi^s \otimes \pi^f) = \inf_{E_{uco}} L$ and $(\pi^s, \pi^f)$ 
  is a solution of the problem \ref{eq:ucoot_copy}. \qed
\end{proof}
%%%%%%%%%%%%%%%%%%%%%%%%%%%%%%%%%%%%%%%%%%%%%

%%%%%%%%%%%%%%%%%%%%%%%%%%%%%%%%%%%%%%%%%%%
\begin{claim}
Suppose that $\mathbb X_1$ and $\mathbb X_2$ are two finite sample-feature spaces such that $(\mathcal X^s_1, \mathcal X^s_2)$ and $(\mathcal X^f_1, \mathcal X^f_2)$ have the same cardinality and are equipped with the uniform measures $\mu_1^s = \mu_2^s$, $\mu_1^f = \mu_2^f$. Then $\ucoot_{\lambda}(\mathbb X_1, \mathbb X_2) = 0$ if and only if there exist perfect alignments between rows (samples) and between columns (features) of the interaction matrices $\xi_1$ and $\xi_2$.
\end{claim}
%%%%%%%%%%%%%%%%%%%%%%%%%%%%%%%%%%%%%%%%%%%%%
\begin{proof}
    Without loss of generality, we can assume that $\mu_k^s$ and $\mu_k^f$ are discrete uniform probability distributions, for $k=1,2$. By Proposition 1 in \citep{Redko20}, under the assumptions on $\mathbb X_1$ and $\mathbb X_2$, we have  $\coot(\mathbb X_1, \mathbb X_2) = 0$ if and only if there exist perfect alignments between rows (samples) and between columns (features) of the interaction matrices $\xi_1$ and $\xi_2$. So, it is enough to prove that $\ucoot_{\lambda}(\mathbb X_1, \mathbb X_2) = 0$ if and only if $\coot(\mathbb X_1, \mathbb X_2) = 0$.
    
    Let $(\pi^s, \pi^f)$ be a pair of equal-mass couplings such that $\ucoot_{\lambda}(\mathbb X_1, \mathbb X_2) = 0$. It follows that $\pi^s_{\#k} \otimes \pi^f_{\#k} = \mu_k^s \otimes \mu_k^f$, for $k=1,2$. Consequently, $m(\pi^s) m(\pi^f) = m(\mu_1^s) m(\mu_1^f) = 1$, so $m(\pi^s) = m(\pi^f) = 1$. Now, we have
    $\int_{\mathcal X_k^s} \mathrm d \pi^s_{\#k} \; \mathrm d \pi^f_{\#k} = \int_{\mathcal X_k^s} \mathrm d\mu_k^s \; \mathrm d\mu_k^f$, or equivalently, $\pi^f_{\#k} = \mu_k^f$. Similarly, $\pi^s_{\#k} = \mu_k^s$, meaning that 
  $\pi^s \in U(\mu_1^s, \mu_2^s)$ and $\pi^f \in U(\mu_1^f, \mu_2^f)$. Thus,
  $\coot(\mathbb X_1, \mathbb X_2) = \ucoot_{\lambda}(\mathbb X_1, \mathbb X_2) = 0$.

  For the other direction, suppose that $\coot(\mathbb X_1, \mathbb X_2) = 0$. Let $(\pi^s, \pi^f)$ be a pair of couplings such 
  that $\coot(\mathbb X, \mathbb Y) = 0$. As $\pi^s \in U(\mu_1^s, \mu_2^s)$ and $\pi^f \in U(\mu_1^f, \mu_2^f)$, one has 
  $\coot(\mathbb X_1, \mathbb X_2) = F_{\lambda}(\pi^s, \pi^f) \geq \ucoot_{\lambda}(\mathbb X_1, \mathbb X_2) \geq 0$, 
  for every $\lambda_1, \lambda_2 > 0$. So, $\ucoot_{\lambda}(\mathbb X_1, \mathbb X_2) = 0$. \qed
\end{proof}

%%%%%%%%%%%%%%%%%%%%%%%%%%%%%%%%%%%ù
\subsection{Robustness of UCOOT and sensitivity of COOT}

First, we recall our assumptions.
%%%%%%%%%%%%%%%%%%%%%%%%%%%%%%%%%%%%%%%%%
\begin{assumption}
\label{assump:robust_copy}
Consider two sample-feature spaces $\bbX_k = ((\mathcal X^s_k, \mu^s_k), (\mathcal X^f_k, \mu^f_k), \xi_k)$, for $k=1,2$. Let $\varepsilon^s$ (resp. $\varepsilon^f$) be a probability measure with compact support $\cO^s$ (resp. $\cO^f$). For $a \in \{s, f\}$, define the noisy distribution $\widetilde{\mu}^a = \alpha_a \mu^a + (1-\alpha_a) \varepsilon^a$, where $\alpha_a \in [0,1]$. We assume that $\xi_1$ is defined on $(\mathcal X^s_1 \cup \cO^s) \times (\mathcal X^f_1 \cup \cO^f)$ and that $\xi_1, \xi_2$ are continuous on their supports. We denote the contaminated sample-feature space by $\widetilde{\mathbb X_1} = ((\mathcal X^s_1 \cup \cO^s, \widetilde{\mu}^s_1), (\cX^f_1 \cup \cO^f, \widetilde{\mu}^f_1), \xi_1)$. Finally, we define some useful minimal and maximal costs:
  \[
  \begin{cases}
  \Delta_{0} \eqdef& \min_{
  \substack{
       x_1^s \in \cO^s, x_1^f \in \cO^f  \\
       x_2^s \in \cX_2^s, x_2^f \in \cX_2^f
  }}\quad |\xi_1(x_1^s, x_1^f) - \xi_2(x_2^s, x_2^f)|^p \\
  \Delta_{\infty} \eqdef& \max_{
  \substack{
  x_1^s \in \cX_1^s \cup \cO^s, x_1^f \in \cX_1^f \cup\cO^f \\
  x_2^s \in \cX_2^s, x_2^f \in \cX_2^f
  }} \quad|\xi_1(x_1^s, x_1^f) - \xi_2(x_2^s, x_2^f)|^p \enspace.
  \end{cases}
\]
\end{assumption}
For convenience, we write $C \eqdef \vert \xi_1 - \xi_2 \vert^p$ and $\widetilde{\mathcal S} := (\mathcal X^s_1 \cup \cO^s) \times \mathcal X_2^s \times (\mathcal X^f_1 \cup \cO^f) \times \mathcal X_2^f$.
%%%%%%%%%%%%%%%%%%%%%%%%%%%%%%%%
\begin{proposition} (COOT is sensitive to outliers)
Consider $\widetilde{\mathds X_1}, \mathds X_2$ as defined in Assumption \ref{assump:robust_copy}. Then:
 \label{prop:coot-not-robust_copy}
\begin{equation*}
 \label{eq:coot-not-robust}
    \coot(\widetilde{\mathds X_1}, \mathds X_2) \geq (1 - \alpha_s)(1-\alpha_f)\Delta_0.
\end{equation*}
%%%%%%%%%%%%%%%%%%%%%%%%%%%%%%%%%
\end{proposition}
\begin{proof}
Consider a pair of feasible alignments $(\pi^s, \pi^f)$. Since $C$ is non-negative, taking the COOT integral over a smaller set leads to the lower bound:
% \coot(\cX_1, \widetilde{\cX_2}) &= \min_{\substack{\pi^s \in \cU(\mu_1^s, \widetilde{\mu_2^s}) \\ \pi^f \in \cU(\mu_1^f, \widetilde{\mu_2^f})}}
  \begin{equation*}
      \begin{split}
          \int_{\widetilde{\cS}} C \mathrm d\pi^s\mathrm d\pi^f 
          &\geq \int_{\cO^s \times \cX_2^s \times \cO^f \times \cX_2^f} C \mathrm d\pi^s\mathrm d\pi^f \\
            &\geq  \Delta_0 \int_{\cO^s \times \cX_2^s \times \cO^f \times \cX_2^f}  \mathrm d\pi^s\mathrm d\pi^f \\
            &= \Delta_0 \int_{\cO^s\times \cO^f}  \mathrm d\pi^s_{\#1} \mathrm d\pi^f_{\#1} \\
            &\geq (1 - \alpha_s)(1 -\alpha_f)\Delta_0,
      \end{split}
  \end{equation*}
  where the last inequality follows from the marginal constraints.
  \qed
\end{proof}

%%%%%%%%%%%%%%%%%%%%%%%%%%%%%%%%%
\begin{theorem}
\label{thm:ucoot_robust_copy}
(UCOOT is robust to outliers)
Consider two sample-feature spaces $\widetilde{\mathds X_1}, \mathds X_2$ as defined in Assumption \ref{assump:robust_copy}. Let $\delta \eqdef 2(\lambda_1 + \lambda_2)(1 - \alpha_s\alpha_f)$ and $K = M + \frac{1}{M}\ucoot(\mathbb X_1, \mathbb X_2) +\delta$, where $M= m(\pi^s) = m(\pi^f)$ is the transported mass between clean data. Then:
   \begin{equation*} %\label{eq:ucoot-robust}
    \begin{split}
      \ucoot(\widetilde{\mathbb X_1}, \mathbb X_2)
      \leq \alpha_s \alpha_f \ucoot(\mathbb X_1, \mathbb X_2)
      + 
      \delta M \left[ 1 - \exp \left( {- \frac{\Delta_{\infty}(1+M) + K}{\delta M}} \right) \right].
    \end{split}
  \end{equation*}
\end{theorem}
%%%%%%%%%%%%%%%%%%%%%%%%%%%%%%%%%%%%%%%%%%%%
To get the exponential bound of this theorem, we use the following lemma.
\begin{lemma}
\label{slem:bound}
Let $\varphi: t \in (0, 1] \mapsto t\log(t) - t + 1$ and $f_{a, b}: t \in (0, 1] \mapsto t \to at + b \varphi(t)$ for some $a, b > 0$. 
Then:
\begin{equation*}
    \min_{t \in (0, 1]} f_{a, b}(t) = b(1 - e^{-a/b}) = f_{a, b}(e^{-\frac{a}{b}}).
\end{equation*}
\end{lemma}
\begin{proof}
  Since $f_{a,b}$ is convex, cancelling the gradient is sufficient for optimality. The solution follows immediately. \qed
\end{proof}
\begin{proof}
  The proof uses the same core idea of \citep{Fatras21} but is slightly more technical for two reasons: (1) we consider arbitrary outlier distributions instead of simple Diracs; (2) we consider sample-feature outliers which requires more technical derivations. 
  
  The idea of proof is as follows. First, we construct sample and feature couplings from the solution of "clean" UCOOT and the reference measures. Then, they are used to upper bound the "noisy" UCOOT. By manipulating this bound, the "clean" UCOOT term will appear. A variable $t \in (0,1)$ is also introduced in the fabricated couplings. The upper bound becomes a function of $t$ and can be optimized to obtain the final bound.
  
  Now, we prove Theorem 2.
  
  \paragraph{Fabricating sample and feature couplings.} Given the equal-mass solution $(\pi^s, \pi^f)$ of the UCOOT problem, with $m(\pi^s) = m(\pi^f) = M$, consider, for $t \in (0,1)$, a pair of sub-optimal transport plans:
  \begin{align*}
    &\widetilde{\pi}^s = \alpha_s \pi^s + t (1-\alpha_s) \varepsilon_s \otimes \mu^s_2\\
    &\widetilde{\pi}^f = \alpha_f \pi^f + t (1-\alpha_f) \varepsilon_f \otimes \mu^f_2.
  \end{align*}
  Then, for $a\in \{s, f\}$, it holds:
  \begin{itemize}
    \item $\widetilde{\pi}^a_{\#1} = \alpha_k \pi^a_{\#1} + t (1 - \alpha_a) \varepsilon_a$, 
    \item $\widetilde{\pi}^a_{\#2} = \alpha_k \pi^a_{\#2} + t (1 - \alpha_a) \mu^a_2$, 
    \item $m(\widetilde{\mu}^a_1) = 1$ and $m(\widetilde{\pi}^a) = \alpha_a M + (1-\alpha_a) t$.
  \end{itemize}
  \paragraph{Establishing and manipulating the upper bound.} Denote $q \eqdef (1 - \alpha_s)(1 - \alpha_f), s \eqdef \alpha_s (1-\alpha_f) + \alpha_f (1 - \alpha_s)$ and recall that on $\widetilde{\cS}$, the cost $C$ is upper bounded by $\Delta_{\infty} = \max_{\widetilde{\cS}}|\xi_1 - \xi_2|^p$. First we upper bound the transportation cost:
  \begin{equation*}
    \label{seq:cost-split}
    \begin{split}
      &\int_{\widetilde{\mathcal S}} C \; \mathrm d\widetilde{\pi}^s \; \mathrm d\widetilde{\pi}^f \\
      &= \alpha_s\alpha_f\int_{\widetilde{\cS}} C \; \mathrm d\pi^s \; d\pi^f + 
      t \sum_{k \neq i} (1-\alpha_i) \alpha_k \int_{\widetilde{\mathcal S}} C \; \mathrm d \varepsilon_i \; \mathrm d\mu^i_2  \; \mathrm d\pi^k +
      q t^2 \int_{\widetilde{\cS}} C \; \mathrm d \varepsilon_s \; \mathrm d\mu_2^s \; \mathrm d\varepsilon_f \; \mathrm d\mu^f_2 \\
      &\leq \alpha_s\alpha_f \int_{\mathcal S} C \; \mathrm d\pi^s \;\mathrm d\pi^f +
      \Delta_{\infty}(Ms + q)t\enspace,
    \end{split}
  \end{equation*}
 since $t^2 \leq t$.

Second, we turn to the KL marginal discrepancies. We would like to extract the KL terms involving only the clean transport plans from the contaminated ones. We first detail both joint KL divergences for the source measure indexed by 1. The same holds for the target measure:
  \begin{equation}
  \label{seq:kl-split}
  \begin{split}
         &\kl(\widetilde{\pi}^s_{\#1} \otimes \widetilde{\pi}^f_{\#1} \vert \widetilde{\mu}^s_1 \otimes \widetilde{\mu}^f_1) = 
    \sum_{k \neq i} m(\widetilde{\pi}^i) \kl(\widetilde{\pi}^k_{\#1} \vert \widetilde{\mu}^k_1) + 
    \prod_{k} \big( m(\widetilde{\pi}^k) - 1 \big)\\
    &
    \kl(\pi^s_{\#1} \otimes \pi^f_{\#1} \vert \mu^s_1 \otimes \mu^f_1) = 
  M \sum_k \kl(\pi^k_{\#1} \vert \mu^k_1) + (M-1)^2.
    \end{split}
  \end{equation}
  Now we upper bound each smaller KL term using the joint convexity of the KL divergence:
  \begin{equation*}
    \begin{split}
      \kl(\widetilde{\pi}^k_{\#1} \vert \widetilde{\mu}^k_1) &\leq
      \alpha_k \kl(\pi^k_{\#1} \vert \mu^k_1) + (1 - \alpha_k) \kl(t \varepsilon_k \vert \varepsilon_k) \\
      &= \alpha_k \kl(\pi^k_{\#1} \vert \mu^k_1) + (1 - \alpha_k) \varphi(t),
    \end{split}
  \end{equation*}
  where $\varphi(t) = t \log t - t + 1$, for $t > 0$. Thus, for $k\neq i$:
  \begin{equation*}
    \begin{split}
      &m(\widetilde{\pi}^i) \kl(\widetilde{\pi}^k_{\#1} \vert \widetilde{\mu}^k_1)
      \leq m(\widetilde{\pi}^i) \alpha_k \kl(\pi^k_{\#1} \vert \mu^k_1) + m(\widetilde{\pi}^i) (1 - \alpha_k) \varphi(t) \\
      &= \alpha_i\alpha_k M \kl(\pi^k_{\#1} \vert \mu^k_1) + t (1-\alpha_i) \alpha_k \kl(\pi^k_{\#1} \vert \mu^k_1) + 
      \alpha_i(1-\alpha_k) M \varphi(t) + t q\varphi(t).
    \end{split}
  \end{equation*}
  Summing over $f$ and $s$, we obtain:
  \begin{equation*}
    \begin{split}
      &\sum_{k \neq i} m(\widetilde{\pi}^i) \kl(\widetilde{\pi}^k_{\#1} \vert \widetilde{\mu}^k_1) \\ 
      &\leq \alpha_s\alpha_f M \sum_k \kl(\pi^k_{\#1} \vert \mu^k_1) + 
      t \sum_{k \neq i} (1-\alpha_i) \alpha_k \kl(\pi^k_{\#1} \vert \mu^k_1) + M s \varphi(t)+ 2q t \varphi(t) \\
      &\leq (\alpha_s\alpha_f + \frac{t s}{M})\left(\kl(\pi^s_{\#1} \otimes \pi^f_{\#1} \vert \mu^s_1 \otimes \mu^f_1) - (1-M)^2\right) + M s \varphi(t)+ 2q t \varphi(t).
      \end{split}
  \end{equation*}
   where, in the last bound,  we used the second equation of \eqref{seq:kl-split} and the fact that $\alpha_s(1-\alpha_f) \leq s$ and  $\alpha_f(1-\alpha_s) \leq s$.
The product of masses of \eqref{seq:kl-split} can be written:
  \begin{equation*}
    \begin{split}
      \prod_{k} \big( m(\widetilde{\pi}^k) - 1 \big) &= \prod_k \big( \alpha_k(M-1) + (1-\alpha_k)(t-1) \big) \\
      &= \alpha_s\alpha_f(1-M)^2 + s(1-M)(1-t) + q(1-t)^2.
    \end{split}
  \end{equation*}
  Thus, combining these upper bounds for the source measure:
  \begin{equation*}
    \begin{split}
      \kl(\widetilde{\pi}^s_{\#1} \otimes \widetilde{\pi}^f_{\#1} \vert \widetilde{\mu}^s_1 \otimes \widetilde{\mu}^f_1) 
      &\leq \alpha_s\alpha_f \kl(\pi^s_{\#1} \otimes \pi^f_{\#1} \vert \mu^s_1 \otimes \mu^f_1) \\
      &+ \frac{ts}{M}\left(\kl(\pi^s_{\#1} \otimes \pi^f_{\#1} \vert \mu^s_1 \otimes \mu^f_1) - (1-M)^2\right) \\
      &+ \big[  sM \varphi(t)+ 2q t \varphi(t) + s(1-M)(1-t) + q(1-t)^2 \big],
    \end{split}
  \end{equation*}
  and similarly, for the target measure:
  \begin{equation*}
    \begin{split}
      \kl(\widetilde{\pi}^s_{\#2} \otimes \widetilde{\pi}^f_{\#2} \vert \mu^s_2 \otimes \mu^f_2) 
      &\leq \alpha_s\alpha_f \kl(\pi^s_{\#2} \otimes \pi^f_{\#2} \vert \mu^s_2 \otimes \mu^f_2) \\
      &+ \frac{ts}{M}\left(\kl(\pi^s_{\#2} \otimes \pi^f_{\#2} \vert \mu^s_2 \otimes \mu^f_2) - (1-M)^2\right) \\
     &+ \big[ sM \varphi(t)+ 2q t \varphi(t) + s(1-M)(1-t) + q(1-t)^2 \big].
    \end{split}
  \end{equation*}
  Then, for every $0 < t \leq 1$, by summing all bounds:
  \begin{equation*}
    \begin{split}
      \ucoot(\widetilde{\mathds X_1}, \mathds X_2) &\leq \alpha_s\alpha_f \ucoot(\mathds X_1, \mathds X_2) + 
      \Delta_{\infty}(Ms + q)t \\
      &+ \frac{ts}{M}(\ucoot(\mathds X_1, \mathds X_2) - (\lambda_1 + \lambda_2)(1-M)^2) \\
      &+ (\lambda_1 + \lambda_2) \big[ s M \varphi(t) + 2q t \varphi(t) + s(1-M)(1-t) + q(1-t)^2 \big].
    \end{split}
  \end{equation*}
  \paragraph{Minimizing the upper bound with respect to $t$.} To obtain the exponential bound, we would like have an upper bound of the form $at + b\varphi(t)$, so that lemma \ref{slem:bound} applies. Knowing that $1 \leq 2(t + \varphi(t))$ for any $t \in [0, 1]$:
  
  Let's first isolate the quantity that is not of this form:
  We have:
  \begin{equation*}
    \begin{split}
      2q t \varphi(t) + s(1 - M) + q(t-1)^2 &= 2qt^2\log(t) - 2qt^2 + 2qt + s(1-M) + qt^2 -2qt + q \\
      &=  2qt^2\log(t) - qt^2 + s(1-M) + q \\
      &= q\varphi(t^2) + s(1-M) \leq q + s(1-M) \\
      &\leq 2(q + s(1-M)) (t + \varphi(t)) \\
      &= 2(1 -\alpha_s\alpha_f - sM) (t + \varphi(t)).
    \end{split}
  \end{equation*}
  The new full bound is given by:
  \begin{equation*}
      \ucoot(\widetilde{\mathds X_1}, \mathds X_2) \leq \alpha_s\alpha_t \ucoot(\mathds X_1, \mathds X_2) + A' t + B'\varphi(t),
  \end{equation*}
  where
  \begin{equation*}
      \begin{split}
          A' &\eqdef \Delta_{\infty}(Ms + q) + s(M-1) + \frac{s}{M}\ucoot(\mathds X_1, \mathds X_2) - \frac{s}{M}(\lambda_1 + \lambda_2)(1-M)^2 \\
          &+ 2(\lambda_1 + \lambda_2) (1-\alpha_s\alpha_f - sM) \\
          & \leq \Delta_{\infty}(M + 1) + M + \frac{1}{M}\ucoot(\mathds X_1, \mathds X_2) + 2(\lambda_1 + \lambda_2) (1-\alpha_s\alpha_f) \eqdef A \\
          B' &\eqdef 2sM(\lambda_1 + \lambda_2) (1-\alpha_s\alpha_f) \leq 2M(\lambda_1 + \lambda_2) (1-\alpha_s\alpha_f) \eqdef B.
      \end{split}
  \end{equation*}
  In both inequalities, we use the fact that $s \leq 1 - \alpha_s \alpha_f \leq 1$. Using Lemma \ref{slem:bound}, we obtain
  \begin{equation*}
      \ucoot(\widetilde{\mathds X_1}, \mathds X_2)
      \leq \alpha_s\alpha_f \ucoot(\mathbb X, \mathbb Y) + B \left[ 1 - \exp{ \left(- \frac{A}{B} \right) }\right].
  \end{equation*}
  The upper bound of Theorem \ref{thm:ucoot_robust_copy} then follows. \qed
\end{proof}
%%%%%%%%%%%%%%%%%%%%%%%%%%%%%%%%%%%%%%%%%%%%%

\subsection{Numerical aspects}
We claim that, in the discrete setting, by taking $\varepsilon$ sufficiently small in the entropic UCOOT problem, we can obtain a solution ``close'' to the non-entropic case. We formalize this claim and prove it in the following result.
%%%%%%%%%%%%%%%%%%%%%%%%%%%%%%%%%%%%%%%%%%
\begin{claim} \label{convergence_minimiser_unbalanced}
  Let $(\pi_{\varepsilon}^s, \pi_{\varepsilon}^f)$ be an equal-mass solution of the problem 
  $\ucoot_{\lambda, \varepsilon}(\mathbb X_1, \mathbb X_2)$. Denote $\mu^s = \mu_1^s \otimes \mu_2^s$ and $\mu^f = \mu_1^f \otimes \mu_2^f$.
  \begin{enumerate}
    \item When $\varepsilon \to \infty, \pi_{\varepsilon}^s \rightharpoonup \sqrt{\frac{m(\mu^f)}{m(\mu^s)}} \mu^s$ 
    and $\pi_{\varepsilon}^f  \rightharpoonup \sqrt{\frac{m(\mu^s)}{m(\mu^f)}} \mu^f$.

    \item When $\varepsilon \to 0$, if the spaces $\mathcal X_k^s$ and $\mathcal X_k^f$ are finite, 
    for $k=1,2$, then $\ucoot_{\lambda, \varepsilon}(\mathbb X_1, \mathbb X_2) \to \ucoot_{\lambda}(\mathbb X_1, \mathbb X_2)$ and 
    any cluster point $\widehat{\pi}^s \otimes \widehat{\pi}^f$ of the sequence 
    $(\pi_{\varepsilon}^s \otimes \pi_{\varepsilon}^f)_{\varepsilon}$ will induce an equal-mass 
    solution $(\widehat{\pi}^s, \widehat{\pi}^f)$ of the problem 
    $\ucoot_{\lambda}(\mathbb X_1, \mathbb X_2)$. Furthermore,
    \begin{equation*}
      \kl(\widehat{\pi}^s \otimes \widehat{\pi}^f | \mu^s \otimes \mu^f) = 
      \min_{(\pi^s, \pi^f)} \kl(\pi^s \otimes \pi^f \vert \mu^s \otimes \mu^f),
    \end{equation*}
    where the infimum is taken over all equal-mass solutions of $\ucoot_{\lambda}(\mathbb X_1, \mathbb X_2)$.
  \end{enumerate}
\end{claim}
%%%%%%%%%%%%%%%%%%%
\begin{proof}
  Denote $\pi_{\varepsilon} = \pi_{\varepsilon}^s \otimes \pi_{\varepsilon}^f$.
  \begin{enumerate}
    \item When $\varepsilon \to \infty$: the sub-optimality of $\left( \sqrt{\frac{m(\mu^f)}{m(\mu^s)}} \mu^s, \sqrt{\frac{m(\mu^s)}{m(\mu^f)}} \mu^f \right)$ implies
    \begin{equation*}
      \begin{split}
        \varepsilon \kl(\pi_{\varepsilon} \vert \mu^s \otimes \mu^f) 
        &\leq F_{\lambda}(\pi_{\varepsilon}^s, \pi_{\varepsilon}^f) + 
        \varepsilon \kl(\pi_{\varepsilon} \vert \mu^s \otimes \mu^f) \\
        &\leq F_{\lambda} \left( \sqrt{\frac{m(\mu^f)}{m(\mu^s)}} \mu^s, \sqrt{\frac{m(\mu^s)}{m(\mu^f)}} \mu^f \right) + 
        \varepsilon \kl( \mu^s \otimes \mu^f \vert \mu^s \otimes \mu^f) \\
        &= \iint \vert \xi_1 - \xi_2 \vert^p \mathrm d\mu^s \mathrm d\mu^f.
      \end{split}
    \end{equation*}
    Thus,
    \begin{equation*}
      0 \leq \kl(\pi_{\varepsilon} \vert \mu^s \otimes \mu^f) 
      \leq \frac{1}{\varepsilon} \iint \vert \xi_1 - \xi_2 \vert^p 
      \mathrm d\mu^s \mathrm d\mu^f \to 0,
    \end{equation*}
    whenever $\varepsilon \to \infty$. We deduce that $\kl(\pi_{\varepsilon} \vert \mu^s \otimes \mu^f)$, 
    thus $\pi_{\varepsilon} \rightharpoonup \mu^s \otimes \mu^f$. The conclusion then follows.

    \item Let $(\pi_*^s, \pi_*^f)$ be a solution of 
    $\ucoot_{\lambda}(\mathbb X_1, \mathbb X_2)$. 
    The optimality of $(\pi_{\varepsilon}^s, \pi_{\varepsilon}^f)$ implies
    \begin{equation*}
      \begin{split}
        \ucoot_{\lambda}(\mathbb X_1, \mathbb X_2) 
      &\leq \ucoot_{\lambda}(\mathbb X_1, \mathbb X_2) + 
      \varepsilon \kl(\pi_*^s \otimes \pi_*^f \vert \mu^s \otimes \mu^f).
      \end{split}
    \end{equation*}
    Thus, when $\varepsilon \to 0$, one has 
    $\ucoot_{\lambda, \varepsilon}(\mathbb X_1, \mathbb X_2) \to 
    \ucoot_{\lambda}(\mathbb X_1, \mathbb X_2)$. Now, for every $\varepsilon > 0$,
    \begin{equation*}
      \begin{split}
        \langle C, \mu^s \otimes \mu^f \rangle &= 
        F_{\lambda} \left( \sqrt{\frac{m(\mu^f)}{m(\mu^s)}} \mu^s, \sqrt{\frac{m(\mu^s)}{m(\mu^f)}} \mu^f \right) + 
        \varepsilon \kl( \mu^s \otimes \mu^f \vert \mu^s \otimes \mu^f ) \\
        &\geq F_{\lambda}(\pi_{\varepsilon}^s, \pi_{\varepsilon}^f) + 
        \varepsilon \kl(\pi_{\varepsilon}^s \otimes \pi_{\varepsilon}^f \vert \mu^s \otimes \mu^f) \\
        &\geq F_{\lambda}(\pi_{\varepsilon}^s, \pi_{\varepsilon}^f).
      \end{split}
    \end{equation*}
    On the other hand, following the same proof in Proposition~\ref{eq:ucoot_existence_copy}, we can show that if 
    $m(\pi_{\varepsilon}) \to \infty$, then $F_{\lambda}(\pi_{\varepsilon}^s, \pi_{\varepsilon}^f) \to \infty$, which 
    contradicts the above inequality. So, there exists $M > 0$ such that $m(\pi_{\varepsilon}) \leq M$, 
    for every $\varepsilon > 0$. 
    
    The set $\widetilde{E}_{uco} = \{\pi \in \mathcal M^+(\mathcal S): m(\pi) \leq M\} \cap E_{uco}$ is clearly compact, 
    thus from the sequence of minimisers $(\pi_{\varepsilon})_{\varepsilon} \subset \widetilde{E}_{uco}$ 
    (i.e. $\pi_{\varepsilon} = \pi_{\varepsilon}^s \otimes \pi_{\varepsilon}^f$), we can extract a 
    converging subsequence $(\pi_{\varepsilon_n})_{\varepsilon_n}$ such that 
    $\pi_{\varepsilon_n} \to \widehat{\pi} = \widehat{\pi}^s \otimes \widehat{\pi}^f \in \widetilde{E}_{uco}$, 
    with $m(\widehat{\pi}^s) = m(\widehat{\pi}^f)$.
    The continuity of the divergences implies that,
    $F_{\lambda, \varepsilon}(\pi_{\varepsilon_n}^s, \pi_{\varepsilon_n}^f) \to 
    F_{\lambda}(\widehat{\pi}^s, \widehat{\pi}^f)$, when $\varepsilon \to 0$. We deduce that 
    $\ucoot_{\lambda}(\mathbb X_1, \mathbb X_2) = F_{\lambda}(\widehat{\pi}^s, \widehat{\pi}^f)$, 
    or equivalently $(\widehat{\pi}^s, \widehat{\pi}^f)$ 
    is a solution of $\ucoot_{\lambda}(\mathbb X_1, \mathbb X_2)$. Moreover, we have
    \begin{equation} \label{unbalanced_max_ent}
      \begin{split}
        0 &\leq F_{\lambda}(\pi_{\varepsilon_n}^s, \pi_{\varepsilon_n}^f) - F_{\lambda}(\pi_*^s, \pi_*^f) \\
      &\leq \varepsilon_n \Big( \kl(\pi_*^s \otimes \pi_*^f \vert \mu^s \otimes \mu^f) - 
      \kl(\pi_{\varepsilon_n}^s \otimes \pi_{\varepsilon_n}^f \vert \mu^s \otimes \mu^f) \Big).
      \end{split}
    \end{equation}
    Dividing by $\varepsilon_n$ in \ref{unbalanced_max_ent} and let $\varepsilon_n \to 0$, we have
    \begin{equation*}
      \kl(\widehat{\pi}^s \otimes \widehat{\pi}^f \vert \mu^s \otimes \mu^f) \leq 
      \kl(\pi_*^s \otimes \pi_*^f \vert \mu^s \otimes \mu^f).
    \end{equation*}
    and we deduce that 
    \begin{equation*}
      \kl(\widehat{\pi}^s \otimes \widehat{\pi}^f \vert \mu^s \otimes \mu^f) = 
      \min_{(\pi^s, \pi^f)} \kl(\pi^s \otimes \pi^f \vert \mu^s \otimes \mu^f),
    \end{equation*}
    where the infimum is taken over all solutions of 
    $\ucoot_{\lambda}(\mathbb X_1, \mathbb X_2)$. \qed
  \end{enumerate}
\end{proof}
%%%%%%%%%%%%%%%%%%%%%%%%%%%%%%%%%%%%%%%%%%%%%

%%%%%%%%%%%%%%%%%%%%%%%%%%%%%%%%%%%%%%%%%%%%%
\section{Algorithmic details} \label{sec_app:algo}

\subsection{Optimization procedure} \label{subsec_app:algo}
Recall that in discrete form, the UCOOT problem reads
\begin{equation} 
\begin{split}
    \label{eq:ucoot-discrete-2_copy}\small
  \min_{\substack{\pi^s, \pi^f \\ 
  \iffalse \in \bbR_+^{n_1, n_2} \\ \pi^f \in \bbR_+^{d_1, d_2} \\ \fi m(\pi^s) = m(\pi^f)}} \sum_{i, j, k, l} &(\bA_{ik} - \bB_{jl})^2\pi^s_{ij}\pi^f_{kl} + \lambda_1 \kl(\pi^s_{\#1} \otimes \pi^f_{\# 1} \vert u_1 )
  + \lambda_2 \kl(\pi^s_{\# 2} \otimes \pi^f_{\# 2} \vert u_2) \\ 
  &+ \varepsilon \kl( \pi^s \otimes \pi^f \vert \mu^s_1 \otimes \mu_2^s \otimes \mu_1^f \otimes \mu_2^f),
\end{split}
\end{equation}
where $\pi_{\# 1} = (\sum_j \pi_{ij})_i$ and $\pi_{\# 2} = (\sum_i \pi_{ij})_j$. Here $\mu_k = \mu_k^s \otimes \mu_k^f$, for $k=1,2$. By Proposition 4 in \citep{Sejourne20}, for fixed $\pi^f \in \mathbb R^{d_1, d_2}_{\geq 0}$, the minimization in \ref{eq:ucoot-discrete-2_copy} is equivalent to solving the following unbalanced OT problem
\begin{equation} \label{eq:ucoot_uot}
  \begin{split}
    \min_{\pi \in \mathbb R^{n_1, n_2}_{\geq 0}} & \; \langle L_{\varepsilon}, \pi \rangle
    + \lambda_1 m_s \kl(\pi_{\#1} \vert \mu_1^s)
    + \lambda_2 m_s \kl(\pi_{\#2} \vert \mu_2^s)
    + \varepsilon m_s \kl(\pi \vert \mu^s_1 \otimes \mu_2^s),
  \end{split}
\end{equation}
where $m_s = m(\pi^s)$ and
\begin{equation*} \label{ucoot:9}
  L_{\varepsilon} := \int \vert A - B \vert^2 \mathrm d \pi^f + 
  \lambda_1 \langle \log \frac{\pi^f_{\#1}}{\mu_1^f}, \pi^f_{\#1} \rangle + 
  \lambda_2 \langle \log \frac{\pi^f_{\#2}}{\mu_2^f}, \pi^f_{\#2} \rangle + 
  \varepsilon \langle \log \frac{\pi^f}{\mu_1^f \otimes \mu_2^f}, \pi^f \rangle,
\end{equation*}
and $\int \vert A - B \vert^2 \mathrm d \pi^f \in \mathbb R^{n_1, n_2}$ defined by $\int \vert A - B \vert^2 \mathrm d\pi^f = A^{\odot 2} \pi^f_{\# 1} \oplus B^{\odot 2} \pi^f_{\# 2} - 2 A \pi^f B^T$. Here, the notations $\otimes$ and $\oplus$ denote the Kronecker product and sum, respectively. For any matrix $M$, we write $M^{\odot 2} := M \odot M$, where $\odot$ is the element-wise multiplication. The exponential, division and logarithm operations are also element-wise. The scalar product is denoted by $\langle \cdot, \cdot \rangle$.

Now, the problem~\ref{eq:ucoot_uot} is of the form
\begin{equation*}
    \min_{P \geq 0} \; \langle C, P \rangle + \rho_1 \kl(P_{\# 1} \vert \mu) + \rho_2 \kl(P_{\# 2} \vert \nu) + \varepsilon \kl(P \vert \mu \otimes \nu),
\end{equation*}
for $\varepsilon, \rho_1, \rho_2 \geq 0$,
and can be solved using the scaling algorithm \citep{Chizat18b} or non-negative penalized regression (NNPR) \citep{Chapel21}, depending on the values of parameters. The complete approximation schemes can be found in Algo~\ref{app:algo_sinkhorn} and Algo~\ref{app:algo_nnpr}. \\

\begin{algorithm}\captionsetup{labelfont={bf}}
\caption{Scaling algorithm \citep{Chizat18b}}
\label{app:algo_sinkhorn}
\begin{algorithmic}
	\STATE {\bfseries Input:} $\bC \in \bbR^{m, n}, \mu \in \bbR^m_{> 0}, \nu \in \bbR^n_{> 0}$, $(\rho_1, \rho_2) \in [0, \infty]^2, \varepsilon > 0$.
	\STATE Initialize $f$ and $g$.
	\REPEAT
	\STATE Update $f$ by: $f = -\frac{\rho_1}{\rho_1 + \varepsilon} \log \sum_j \exp \big( g_j + \log \nu_j - \frac{C_{\cdot,j}}{\varepsilon} \big)$.
	\STATE Update $g$ by: $g = -\frac{\rho_2}{\rho_2 + \varepsilon} \log \sum_i \exp \big( f_i + \log \mu_i - \frac{C_{i,\cdot}}{\varepsilon} \big)$.
	\UNTIL{convergence}
	\STATE Calculate: $P = (\mu \otimes \nu) \exp \big(f \oplus g - \frac{C}{\varepsilon} \big)$.
\end{algorithmic}
\end{algorithm}
\text{ } \\
\begin{algorithm}\captionsetup{labelfont={bf}}
\caption{Non-negative penalized regression (NNPR) \citep{Chapel21}}
\label{app:algo_nnpr}
\begin{algorithmic}
	\STATE {\bfseries Input:} $\bC \in \bbR^{m, n}, \mu \in \bbR^m_{> 0}, \nu \in \bbR^n_{> 0}$, $(\rho_1, \rho_2) \in [0, \infty)^2, \varepsilon \geq 0$.
	\STATE Calculate $\lambda = \rho_1 + \rho_2 + \varepsilon$, then $r = \frac{\varepsilon}{\lambda}$ and $\lambda_i = \frac{\rho_i}{\lambda}$, for $i=1,2$.
	\STATE Initialize $P$.
	\REPEAT
	\STATE Update $P$ by: $P = \frac{P^{\lambda_1 + \lambda_2}}{P_{\# 1}^{\lambda_1} \otimes P_{\# 2}^{\lambda_2}} 
    \odot \left( \mu^{\lambda_1 + r} \otimes \nu^{\lambda_2 + r} \right) 
    \odot \exp\left(-\frac{C}{\lambda} \right)$.
	\UNTIL{convergence}
\end{algorithmic}
\end{algorithm}

It should be noted that scaling algorithm allows for $\rho_k = \infty$ (so $\frac{\rho_k}{\rho_k + \varepsilon} = 1$ and we recover the usual Sinkhorn update), but $\varepsilon$ must be \textit{strictly} positive. 
On the other hand, NNPR allows for every  $\varepsilon \geq 0$, but both $\rho_1$ and $\rho_2$ must be \textit{both finite}.

\paragraph{Complexity} UCOOT's complexity has similar complexity to the  entropic COOT that had been investigated in the supplementary of \citep{Redko20}. The latter solves two inner entropic OT problems which implies roughly quadratic complexity, which is also similar to the complexity of solving GW. However, we note that UCOOT can benefit from the recent advance in OT, for example \citep{Schmitzer19,Meyer21a}.

\subsection{UCOOT helps finding better minima} 

\begin{wrapfigure}{R}{0.45\textwidth}
\begin{minipage}{0.45\textwidth}
\includegraphics[width=0.8\linewidth]{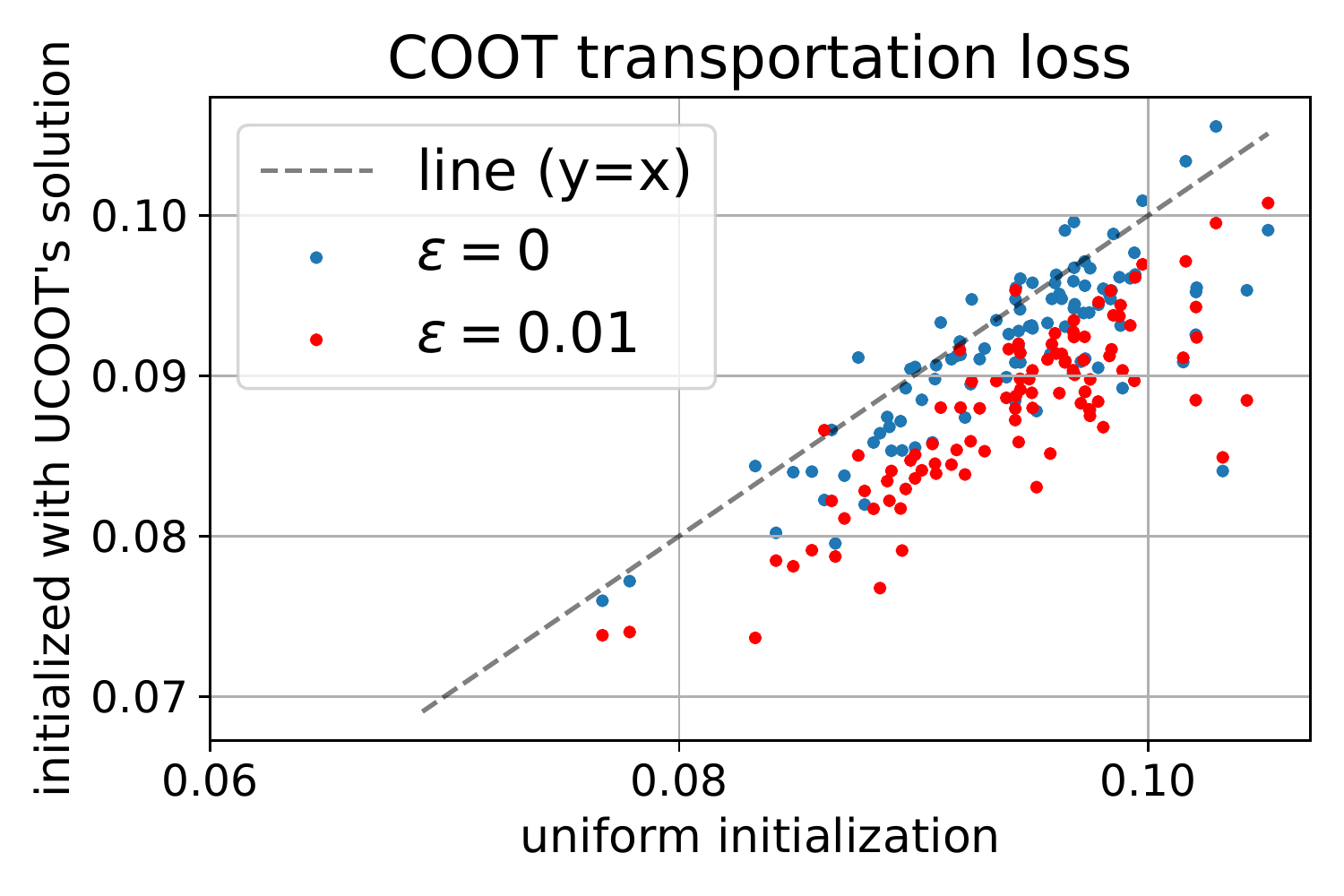}
\captionof{figure}{Scatter plot of the COOT transportation cost with naive (uniform) initialization (y-axis) vs initialization with UCOOT.
\label{fig:init}}
\end{minipage}
\end{wrapfigure}
Interestingly, we find that COOT can achieve better minima when initialized with the UCOOT solutions.  Figure~\ref{fig:init} illustrates how UCOOT can lead to better alignments by finding better local minima of the COOT transportation cost. For 100 random Gaussian datasets $\bA, \bB$ with uniformly sampled shapes $n_1, n_2, d_1, d_2$, we visualize the COOT loss with uniform initialization (y-axis) versus the COOT loss when initializing the COOT BCD with UCOOT's solution. For both $\varepsilon = 0$ and $\varepsilon > 0$, the latter leads to lower COOT costs than the former, on average. 

As COOT is a non-convex problem, the choice of initialization plays an important role. Intuitively, by choosing $\lambda_1$ and $\lambda_2$ sufficiently large, one can use UCOOT to approach COOT. For this reason, the solution of UCOOT can be more informative than the usual uniform initialization, and one can expect to reach better local optimal of COOT.
%  of the form: $\pi \gets \diag(\frac{\ba}{\pi \mathds 1})^\frac{1}{2}\pi\odot(e^{-\frac{\bC}{\rho} })\diag(\frac{\bb}{\pi^\top \mathds 1})^\frac{1}{2}$ for some fixed inputs $\ba, \bb, \bC, \rho$.

%%%%%%%%%%%%%%%%%%%%%%%%%%%%%%%%%%%%%%%%%%%%%
\section{Experimental details} \label{sec_app:exp}

\subsection{More details on barycentric mapping}  \label{bary_mapping}

The barycentric mapping \citep{Ferradans14,Courty16} is a method to transform the source data to the target domain. Given the source data $X_s \in \mathbb R^{n_s \times d_s}$ and target data $X_t \in \mathbb R^{n_t \times d_t}$, once the optimal transportation plan $P \in \mathbb R^{n_s \times n_t}$ is learned, the transformation of the source to the target domain, can be expressed as: for $i = 1, ..., n_s$,
\begin{equation} \label{bary_optim}
    \widehat{x}_i^{(s)} \in \arg\min_{x \in \mathbb R^{d_t}} \sum_{j = 1}^{n_t} P_{ij} \; c(x, x_j^{(t)}),
\end{equation}
where the example $x^{(t)}_j \in \mathbb R^{d_t}$ corresponds to the $j$-th row of $X_t$ and the cost $c: \mathbb R^{d_t} \times R^{d_t} \to \mathbb R$ measures the discrepancy between two examples in $\mathbb R^{d_t}$. Typically, $c$ is the squared Euclidean distance, so the problem \ref{bary_optim} admits a closed form solution, which is a weighted average of examples in the target domain:
\begin{equation}
    \widehat{x}_i^{(s)} = \sum_{j = 1}^{n_t} \frac{P_{ij}}{p_i} x_j^{(t)},
\end{equation}
where $p_i = \sum_{j} P_{ij}$, or in matrix notation:
\begin{equation} \label{transOT:1}
    \widehat{X}_s= \text{diag}(\frac{1}{P 1_{n_t}}) P X_t \in \mathbb R^{n_s \times d_t},
\end{equation}
where the division is element-wise.

\subsection{Heterogenous Domain Adaptation (HDA)}
\paragraph{More details on label propagation}

Once the sample coupling $P$ is learned, the label propagation works as follows: suppose the labels contain $K$ different classes, 
we apply the one-hot encoding to the source label $y^{(s)}$ to obtain 
$D^{(s)} \in \mathbb R^{K \times n_s}$ where $D^{(s)}_{ki} = 1_{\{y^{(s)}_i = k\}}$. The label proportions on the target data are 
estimated by: $L = D^{(s)} P \in \mathbb R^{K \times n_t}$. Then the prediction can be generated by choosing the 
label with the highest proportion, i.e. $\widehat{y}^{(t)}_j = \arg\max_k L_{kj}$.

\paragraph{Paramater validation}

We tune the hyperparameters of each method via grid search.
\begin{itemize}
  \item For COOT, we choose the regularisation on the feature and sample couplings $\varepsilon_f, \varepsilon_s \in \{0, 0.01, 0.1, 0.5\}$.
  \item For GW, we choose the regularisation parameter $\varepsilon \in \{0, 0.001, 0.005, 0.01, 0.05, 0.1, 0.5\}$.
  \item For UGW and UCOOT, we choose $\lambda_1, \lambda_2 \in \{1, 5, 20, 50\}$ and $\varepsilon \in \{0.01, 0.05, 0.1, 0.5\}$.
  Furthermore, for UGW and GW, before calculating the Euclidean distance matrix for each domain, the matrix of domain data is normalised by max scaling, so that its coordinates are bounded in $[-1,1]$. This pre-processing step improves the performance of the for UGW and GW.
\end{itemize}
For each method, for each combination of tuple of hyperparameters, first, we choose a pair amongst $9$ pairs, then repeat $10$ times the training 
procedure, in which the optimal plan is estimated, then used to calculate the accuracy. We choose the tuple of hyperparameters corresponding to the 
highest average accuracy. This optimal tuple is then applied to all other $8$ tasks, where in each task, the training procedure is repeated $10$ times and we report the average accuracy.

\paragraph{When there is no regularization} In the above hyperparameter tuning process, we only considered $\varepsilon > 0$ for UCOOT and UGW, so that the scaling algorithm \citep{Chizat18b} is applicable. As discussed in Section \ref{subsec_app:algo}, the NNPR solver can allow us to handle the case $\varepsilon = 0$ (i.e. we can estimate directly UCOOT, rather than via its entropic approximation). In this case, we also tune $\lambda_1, \lambda_2 \in \{ 1, 50, 20, 50\}$ and follow exactly the same tuning and testing procedure as in the case $\varepsilon > 0$. We report our finding in Table \ref{tab:hda2}. We observe that, in many tasks, the performance remains competitive while enjoying lower variance. \\

\begin{table}[H]
    \centering
    \small
	\begin{tabular}{c c c c c}
		\toprule
		& \multicolumn{2}{c}{CaffeNet $\to$ GoogleNet} \\
		\midrule
		Domains & COOT & UCOOT ($\varepsilon > 0$) & UCOOT ($\varepsilon = 0$) \\
		\midrule
		
		C $\to$ C & 36.40 ($\pm$ 12.94) & \textbf{44.05 ($\pm$ 19.33)} & 38.60 ($\pm$ 9.16) \\
		\hline
		C $\to$ A & 28.30 ($\pm$ 11.78) & \textbf{31.90 ($\pm$ 7.43)} & 29.45 ($\pm$ 9.94) \\
		\hline
		C $\to$ W & 19.55 ($\pm$ 14.51) & 28.55 ($\pm$ 6.60) & \textbf{40.85 ($\pm$ 12.53)} \\
		\hline
		
		A $\to$ C & \textbf{41.80 ($\pm$ 14.81)} & 39.15 ($\pm$ 17.98) & 18.00 ($\pm$ 9.22) \\
		\hline
		A $\to$ A & \textbf{57.90 ($\pm$ 16.84)} & 42.45 ($\pm$ 15.47) & 40.40 ($\pm$ 8.40) \\
		\hline
		A $\to$ W & 42.10 ($\pm$ 7.80) & 48.55 ($\pm$ 13.06) & \textbf{49.15 ($\pm$ 6.64)} \\
		\hline
		
		W $\to$ C & 8.60 ($\pm$ 6.56) & \textbf{69.80 ($\pm$ 14.91)} & 19.70 ($\pm$ 5.79) \\
		\hline
		W $\to$ A & 16.65 ($\pm$ 10.01) & \textbf{30.55 ($\pm$ 10.09)} & 25.90 ($\pm$ 5.48) \\
		\hline
		W $\to$ W & \textbf{75.30 ($\pm$ 3.26)} & 51.50 ($\pm$ 20.51) & 49.55 ($\pm$ 6.02) \\
		\bottomrule
		Average & 36.29 ($\pm$ 10.95) & \textbf{42.94 ($\pm$ 13.93)} & 34.62 ($\pm$ 11.17) \\    
		\bottomrule
	
	\end{tabular}
	\caption{Unsupervised HDA from CaffeNet to GoogleNet for $\varepsilon > 0$ and $\varepsilon = 0$. UCOOT ($\varepsilon > 0$) corresponds to the model where $\varepsilon, \lambda_1$ and $\lambda_2$ are tuned, with $\varepsilon > 0$, and UCOOT ($\varepsilon = 0$) means that $\varepsilon = 0$ and only $\lambda_1, \lambda_2$ are tuned.}
	\label{tab:hda2}
\end{table}

\paragraph{Sensitivity analysis}

We report the sensitivity of UCOOT's performance to the hyper-parameters $\varepsilon, \lambda_1$ and $\lambda_2$ for two tasks C$\to$W and A$\to$A in Tables \ref{tab:sensitiv_eps}, \ref{tab:sensitiv_lambda1} and \ref{tab:sensitiv_lambda_2}, respectively. In general, the performance depends significantly on the choice of hyperparameters. In Table \ref{tab:sensitiv_eps}, given fixed values of $\lambda_1$ and $\lambda_2$, UCOOT performs badly for either too small or large values of $\varepsilon$, indicating that regularization is necessary but should not be too strong. From Table \ref{tab:sensitiv_lambda1}, we see that  large value of $\lambda_1$ degrades the performance, meaning that the marginal constraints on the source distributions should not be too tight. Meanwhile, it seems that large $\lambda_2$ is preferable, so the marginal distributions on the target spaces should not be too relaxed. \\

\begin{table}[H]
	\begin{center}
	\small
		\begin{tabular}{c c c c c c c c c}
			\toprule
			& \multicolumn{7}{c}{CaffeNet $\to$ GoogleNet} \\
			\midrule
			Domains & $\varepsilon= 0.03$ & 0.05 & 0.07 & 0.1 & 0.2 & 0.3 & 0.4 \\
			\midrule
			C $\to$ W & 27.65 ($\pm$ 11.34) & 37.20 ($\pm$ 9.35) & 34.50 ($\pm$ 11.07) & 34.75 ($\pm$ 13.04) & 17.00 ($\pm$ 5.92) & 18.45 ($\pm$ 1.11) & 11.25 ($\pm$ 1.66) \\
			\hline
			A $\to$ A & 21.95 ($\pm$ 9.46) & 35.30 ($\pm$ 15.11) & 35.65 ($\pm$ 15.05) & 41.15 ($\pm$ 19.16) & 58.45 ($\pm$ 15.54) & 22.30 ($\pm$ 3.74) & 8.90 ($\pm$ 1.34) \\
			\hline
		\end{tabular}
	\end{center}
	\caption{Sensitivity of UCOOT to $\varepsilon$ in tasks C$\to$W and A$\to$A. We fix $\lambda_2 = 50$ and $\lambda_1 = 1$ and show the accuracy for various value of $\varepsilon$.} \label{tab:sensitiv_eps}
\end{table}
%%%%%%%%%%%%%%%%%
\begin{table}[H]
    \centering
    \small
	\begin{tabular}{c c c c c c c c c}
		\toprule
		& \multicolumn{7}{c}{CaffeNet $\to$ GoogleNet} \\
		\midrule
		Domains & $\lambda_1= 20$ & 30 & 40 & 50 & 60 & 70 & 80 \\
		\midrule
		C $\to$ W & 35.80 ($\pm$ 9.33) & 34.15 ($\pm$ 12.98) & 37.35 ($\pm$ 13.82) & 27.45 ($\pm$ 8.33) & 32.45 ($\pm$ 11.62) & 30.00 ($\pm$ 8.04) & 30.15 ($\pm$ 12.89) \\
		\hline
		A $\to$ A & 55.20 ($\pm$ 18.44) & 53.35 ($\pm$ 18.74) & 44.15 ($\pm$ 21.54) & 24.30 ($\pm$ 15.58) & 36.10 ($\pm$ 23.97) & 32.35 ($\pm$ 14.88) & 24.80 ($\pm$ 15.08) \\
		\hline
	\end{tabular}
	\caption{Sensitivity of UCOOT to $\lambda_1$ in tasks C$\to$W and A$\to$A. We fix $\lambda_2 = 1$ and $\varepsilon = 0.1$ and show the accuracy for various value of $\lambda_1$.} \label{tab:sensitiv_lambda1}
\end{table}
%%%%%%%%%%%%%%
\begin{table}[H]
    \centering
    \small
	\begin{tabular}{c c c c c c c c c}
		\toprule
		& \multicolumn{7}{c}{CaffeNet $\to$ GoogleNet} \\
		\midrule
		Domains & $\lambda_2= 0.3$ & 0.5 & 0.7 & 1 & 2 & 3 & 4 \\
		\midrule
		C $\to$ W & 34.20 ($\pm$ 9.83) & 34.45 ($\pm$ 10.80) & 34.20 ($\pm$ 10.50) & 34.75 ($\pm$ 13.04) & 29.70 ($\pm$ 10.55) & 37.70 ($\pm$ 17.96) & 32.30 ($\pm$ 18.81) \\
		\hline
		A $\to$ A & 20.75 ($\pm$ 10.11) & 29.00 ($\pm$ 15.79) & 29.25 ($\pm$ 20.66) & 41.15 ($\pm$ 19.16) & 32.65 ($\pm$ 8.80) & 42.10 ($\pm$ 20.71) & 49.95 ($\pm$ 15.75) \\
		\hline
	\end{tabular}
	\caption{Sensitivity of UCOOT to $\lambda_2$ in tasks C$\to$W and A$\to$A. We fix $\lambda_1 = 50$ and $\varepsilon = 0.1$ and show the accuracy for various value of $\lambda_2$.} \label{tab:sensitiv_lambda_2}
\end{table}

\paragraph{Additional results} 
We also perform the adaptation from GoogleNet to CaffeNet. The results can be found in the tables \ref{tab:table3}. We draw the same conclusions as in the adaptation from CaffeNet to GoogleNet.

\begin{table}[H]
    \centering
    \small
    \begin{tabular}{c c c c c}
      \toprule
      & \multicolumn{3}{c}{GoogleNet $\to$ CaffeNet} \\
      \midrule
      Domains & GW & UGW & COOT & UCOOT \\
      \midrule

      C $\to$ C & 19.45 ($\pm$ 10.88) & 17.50 ($\pm$ 4.88) & 46.20 ($\pm$ 14.94) & \textbf{46.50 ($\pm$ 5.81)} \\
      \hline
      C $\to$ A & 9.35 ($\pm$ 7.73) & 10.50 ($\pm$ 7.06) & 33.25 ($\pm$ 17.56) & \textbf{34.45 ($\pm$ 4.89)} \\
      \hline
      C $\to$ W & 19.15 ($\pm$ 10.59) & 11.95 ($\pm$ 7.49) & 14.95 ($\pm$ 12.44) & \textbf{33.60 ($\pm$ 10.07)} \\
      \hline
      
      A $\to$ C & 7.90 ($\pm$ 4.92) & 11.70 ($\pm$ 5.57) & 28.80 ($\pm$ 12.02) & \textbf{40.55 ($\pm$ 6.50)} \\
      \hline
      A $\to$ A & 19.75 ($\pm$ 9.51) & 18.40 ($\pm$ 11.71) & \textbf{59.30 ($\pm$ 20.77)} & 58.95 ($\pm$ 10.37) \\
      \hline
      A $\to$ W & 14.55 ($\pm$ 14.62) & 10.05 ($\pm$ 4.70) & 9.75 ($\pm$ 7.75) & \textbf{65.20 ($\pm$ 9.80)} \\
      \hline

      W $\to$ C & 14.05 ($\pm$ 5.97) & 21.95 ($\pm$ 4.33) & 13.70 ($\pm$ 7.01) & \textbf{33.45 ($\pm$ 6.67)} \\
      \hline
      W $\to$ A & 22.85 ($\pm$ 11.87) & 20.90 ($\pm$ 5.98) & \textbf{47.70 ($\pm$ 5.53)} & 44.45 ($\pm$ 6.02) \\
      \hline
      W $\to$ W & 24.10 ($\pm$ 15.78) & 27.95 ($\pm$ 8.34) & \textbf{72.55 ($\pm$ 4.82)} & 68.80 ($\pm$ 10.24) \\
      \bottomrule
      Average & 16.79 ($\pm$ 10.21) & 16.77 ($\pm$ 6.67) & 36.24 ($\pm$ 11.43) & \textbf{47.33 ($\pm$ 7.82)} \\    
      \bottomrule
    \end{tabular}
  \caption{Unsupervised HDA from GoogleNet to CaffeNet.}
  \label{tab:table3}  
\end{table}

\subsection{Multi-omic dataset alignment}\label{multiOmics_exp_SI}
\paragraph{Data Preprocessing} \label{CITEseq_exp_appendix}
For the single-cell multi-omics experiments, we use the ``PBMC'' dataset from Stoeckius \textit{et al} \cite{CITEseq}, accessed on Gene Expression Omnibus (GEO) with the accession code: \url{https://www.ncbi.nlm.nih.gov/geo/query/acc.cgi?acc=GSE100866}{GSE100866}. This dataset contains a mix of 7,985 mouse and human peripheral blood mononuclear cells (PBMC) and profiles ten antibodies, 17,014 human genes, and 12,915 mouse genes. To pick the human cells, we follow the description in \cite{CITEseq}, and select the cells that have at least 500, and more than 90$\%$ of all unique molecular identifiers (UMIs) mapped to the human genes (rather than the mouse genes). From the resulting $\sim 4500$ cells, we pick the first 1000 to use in our experiments. We use the CLR-normalized antibody count data provided in GEO and apply log normalization to the gene expression data using Seurat package in R to remove biases in sequencing across cells \cite{Seurat}. Prior to alignment, we follow the existing single-cell alignment methods \cite{Demetci22, Demetci22-2,Liu2019, singh2020}, and also apply L2 normalization to both modalities. The top 50 most variable genes (Figure \ref{fig:multiomics}(c)) are selected using the \texttt{FindVariableFeatures()} function from Seurat \cite{Seurat}. 

\paragraph{Hyperparameter tuning} Hyperparameters were tuned using grid search. For both COOT and UCOOT, we considered the following range for the entropic regularization coefficients $\epsilon_f, \epsilon_s \in \{1e-5, 5e-5, 1e-5, 5e-4, ... ,0.1, 0.5\}$. For the mass relaxation coefficients $\lambda_1, \lambda_2 $ in UCOOT, the following range was considered $\lambda_1, \lambda_2 \in \{1e-3, 5e-3, 0.01, 0.05, 0.1, 0.5, 1, 5, 10, 50 ,100\}$. Each combination of hyperparameters were run on three randomly chosen subsets of the dataset that included 30\% of the samples and the hyperparameter combinations that on average yielded the highest feature matches and lowest FOSCTTM were picked for the experiments on the full dataset. Below, we list the hyperparameter combinations used for the final alignment results reported in this paper:
\begin{itemize}
    \item \textbf{Balanced scenario of aligning matching features (Figure ~\ref{fig:multiomics}(a)):} \\ $\lambda_1=1, \lambda_2=0.1, \epsilon_1=1e-4, \epsilon_2=1e-4$
    \item \textbf{Unbalanced scenario of aligning a subset of the matching features (Figure ~\ref{fig:multiomics}(b)):} \\ $\lambda_1=1, \lambda_2=1e-2, \epsilon_1=1e-4, \epsilon_2=1e-4$
      \item \textbf{Unbalanced scenario of aligning antibodies with the top 50 most variable genes (Figure ~\ref{fig:multiomics}(c)):}  $\lambda_1=10, \lambda_2=5e-5, \epsilon_1=1e-4, \epsilon_2=0.5$
      \item \textbf{Balanced scenario of aligning the same number of cells (Figure ~\ref{fig:multiomicsSamples}(a)):} \\ $\lambda_1=1, \lambda_2=0.1, \epsilon_1=1e-4, \epsilon_2=1e-4$
      \item \textbf{Unbalanced scenario 1 of aligning different number of cells (Figure ~\ref{fig:multiomicsSamples}(b)):} \\ $\lambda_1=0.01, \lambda_2=0.1, \epsilon_1=5e-3, \epsilon_2=1e-4$
      \item \textbf{Unbalanced scenario 2 of aligning different number of cells (Figure ~\ref{fig:multiomicsSamples}(c)):} \\ $\lambda_1=0.01, \lambda_2=0.1, \epsilon_1=5e-3, \epsilon_2=1e-4$
\end{itemize}
\paragraph{Further investigation of the feature alignments} 
In the unbalanced experiment, where we align the most variable genes and the antibodies, we expect a well-performing alignment method to correctly match antibodies with the genes that express them. This would be the strongest biological connection between a protein (i.e. an antibody, in this case) and a gene. However, other biological connections can also exist, such as between an antibody and a gene that regulates the expression of that antibody, a gene that codes for a protein the antibody physically interacts with, or a gene that codes for a protein that is active in the same biological pahtway as the antibody of interest. To investigate whether there are such matches recovered outside of the ten genes we label as ``matching genes'', we refer to two gene regulatory network databases that contain data on human PBMCs, GRNdb \cite{GRNdb} and GRAND \cite{GRAND} (for the first kind of relationship),  two protein--protein interaction databases, BioGRID \cite{BIOGRID}, and STRING \cite{STRING} (for the second kind of relationship), and KEGG \cite{KEGG}, a database of biological pathways (for the last kind of relationship). 

Of the 46 correspondences yielded by COOT, and 16 by UCOOT, outside of the `correspondences with `matching genes'', only a few show up on these databases:
\begin{itemize}
    \item \textbf{CD19 antibody correspondences:} Both BIOGRID \cite{BIOGRID} and STRING \cite{STRING} databases return a physical interaction with CD79A protein (encoded by the \textit{CD79A} gene), which is experimentally validated by affinity capture-Western \cite{Carter97}. The correspondence with \textit{CD79A} is yielded by both UCOOT and COOT. Additionally, according to KEGG \cite{KEGG}, CD19 participates in the B-cell receptor (BCR) signaling pathway along with IGH, which is formed by multiple segments joining together, including IGHD and IGHM \footnote{\url{https://www.genecards.org/cgi-bin/carddisp.pl?gene=IGHD}, and \url{https://www.genecards.org/cgi-bin/carddisp.pl?gene=IGH&keywords=IGH}}. COOT yields correspondences with the genes that code for these.
    \item \textbf{CD57 antibody correspondences:} There is an experimentally validated physical interaction with ITM2C (encoded by the \textit{ITM2C} gene), which shows up on BIOGRID. This interaction has been validated using proximity labeling mass spectrometry \cite{Go2021}. \textit{ITM2C} is among the correspondences yielded by COOT.
    \item \textbf{CD2 antibody correspondences:} According to the BIOGRID database, a physical interaction between CD2 and PTPRC has been proposed via an \textit{in vitro} study \textit{et al} \cite{Schraven90}. \textit{PTPRC} shows up among the correspondences yielded by both UCOOT and COOT for the CD2 antibody.
    \item \textbf{CD4 antibody correspondences:} According to BIOGRID, CD4 has been shown to physically interact with TUBB using affinity capture mass spectrometry by Bernhard \textit{et al} \cite{Bernhard2004}. TUBB is a component of the tubulin protein, which made out of $\beta-$tubulin (TUBB) and $\alpha-$tubulin (TUBA). UCOOT yields a correspondence between CD4 and TUBA1B (gene that codes for a component of the $\alpha-$tubulin) \footnote{\url{https://www.nature.com/scitable/content/microtubules-the-basics-14673338/}}.
\end{itemize}

Outside of these, no other correspondences returned biological relevance based on our database and literature search, which leads us to conclude COOT yields more redundant correspondences that UCOOT. 

\paragraph{Sample alignment experiments} Below in Figure \ref{fig:multiomicsSamples}, we visualize the aligned samples (first two principal components of the two domains together upon barycentric projection) and report the alignment performance as measured by the ``average fraction of samples closer than true match (FOSCTTM)'' and ``label transfer accuracy'' metrics. We borrow these metrics from previously published single-cell multi-omic data alignment methods \cite{Liu2019,Cao19, Demetci22, Pamona, Demetci22-2}. For label transfer accuracy, we follow the previously published methods \cite{Cao19, Pamona, Demetci22, Demetci22-2} and train a $k-$NN classifier (for $k=5$) on the cell-type labels of the measurement domain with the full set of cells, and apply it to predict the cell-type labels of the downsampled domain. We report the prediction accuracy. For the balanced scenario, we train the classifier on the antibody domain to predict the labels in the gene expression domain upon transportation. For the average FOSCTTM metric used in unbalanced scenarios, we use the cells that remain to have a correspondence after subsampling to calculate the FOSCTTM scores. We note that lower average FOSCTTM and higher label transfer accuracy results indicate better alignments.

\begin{figure}[H]
    \centering
    \includegraphics[width=\linewidth]{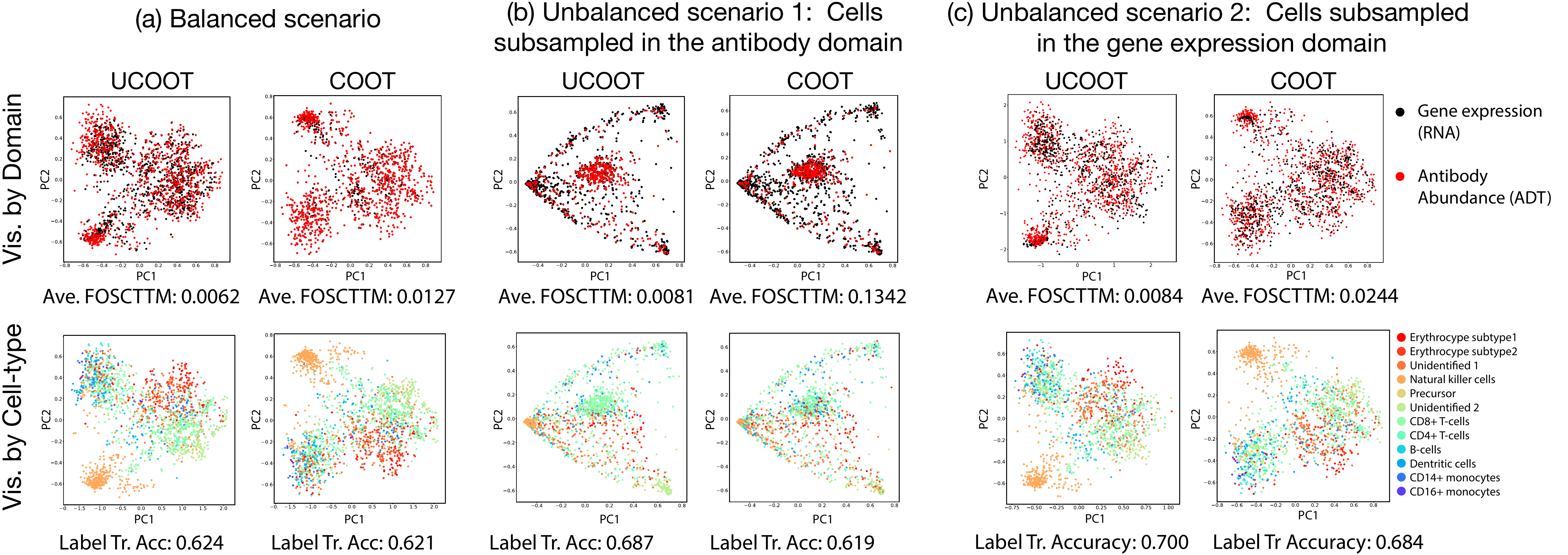}
    \caption{Visualization of the sample alignments with UCOOT and COOT after barycentric projection (First two principal components). The top row visualizes results with samples colored based on measurement modality (black points show the gene expression domain samples, and red points show the antibody domain samples). The bottom row visualizes alignments with samples colored based on cell-type labels. \textbf{(a)} presents sample alignments in the balanced scenario, where we align the same number of cells (1000) in each measurement modality with the matching features (same scenario as Fig 6 (a), but presenting sample alignments). \textbf{(b)} In this unbalanced scenario, we randomly downsample the cells in antibody domain by 25\%. \textbf{(c)} In this second unbalanced scenario, we randomly downsample the cells in the gene expression domain by 25\% and align with the full set of samples in the antibody domain. For all alignments, we quantify alignment quality using average FOSCTTM (``Ave. FOSCTTM'') and label transfer accuracy (``Label Tr. Acc.''), and report them under the plots. We calculate both metrics prior to applying dimensionality reduction with principal component analysis (PCA). PCA is only applied for visualization purposes. Note that the overall increase in label transfer accuracy between \textbf{a-c} is likely due to the removal of groups of heterogenous cell types during downsampling.
\label{fig:multiomicsSamples}
  }
\end{figure}

\end{document}